\documentclass[accepted]{uai2023} 

\usepackage[american]{babel}

\usepackage{natbib} 
\usepackage{mathtools} 
\usepackage{booktabs} 
\usepackage{tikz} 

\usepackage{hyperref}
\usepackage{smile}

\title{
Exact Count of Boundary Pieces of ReLU Classifiers: \\
Towards the Proper Complexity Measure for Classification}

\author[1]{Pawe\l~Piwek}
\author[2]{Adam~Klukowski}
\author[2]{\href{mailto:<hutianyang.up@outlook.com>?Subject=Your UAI 2023 paper}{Tianyang~Hu}}
\affil[1]{%
    University of Oxford\\ 
    \texttt{pawel.piwek@maths.ox.ac.uk}
}
\affil[2]{%
    Huawei Noah's Ark Lab\\
    \texttt{hutianyang1@huawei.com}
}

\usepackage{svg}
\usepackage[capitalize,noabbrev]{cleveref}

\renewcommand{\vec}[1]{\mathbf{#1}}
\newcommand{\x}{\vec{x}}
\newcommand{\y}{\vec{y}}
\newcommand{\R}{\mathbb{R}}

\newcommand{\U}{\mathcal{U}}

\newcommand{\T}{\mathcal{T}}

\usepackage{xcolor}

\counterwithout{theorem}{section}
\crefname{example}{Example}{Examples}
\crefname{proposition}{Proposition}{Propositions}

\begin{document}

\maketitle

\begin{abstract}
Classic learning theory suggests that proper regularization is the key to good generalization and robustness. In classification, current training schemes only target the complexity of the classifier itself, which can be misleading and ineffective. 
Instead, we advocate directly measuring the complexity of the decision boundary. 
Existing literature is limited in this area with few well-established definitions of boundary complexity. As a proof of concept, we start by analyzing ReLU neural networks, whose boundary complexity can be conveniently characterized by the number of affine pieces. With the help of tropical geometry, we develop a novel method that can explicitly count the exact number of boundary pieces, and as a by-product, the exact number of total affine pieces. Numerical experiments are conducted and distinctive properties of our boundary complexity are uncovered. First, the boundary piece count appears largely independent of other measures, e.g., total piece count, and $l_2$ norm of weights, during the training process. Second, the boundary piece count is negatively correlated with robustness, where popular robust training techniques, e.g., adversarial training or random noise injection, are found to reduce the number of boundary pieces. 

\end{abstract}

\section{Background}

Despite deep learning's huge success in image classification, naturally trained deep classifiers are found to be adversarially vulnerable \citep{goodfellow2014explaining, goodfellow2016deep}. 
By adding a small perturbation (adversarial attack) to an image, which is almost imperceptible to humans, the neural network's predicted class can be arbitrarily manipulated. 
The prevalence of adversarial examples for state-of-the-art deep classifiers, even on small datasets such as CIFAR \citep{krizhevsky2009learning}, suggests overfitting, where decision boundaries of trained deep neural networks (DNNs) are \textit{overly complicated} and within a small distance to almost all the training instances. 
Ideally, we want our model to generalize well on unseen data and be robust against small input perturbations, i.e., the prediction doesn't change much in case of small random noises. 
For regression, the requirement loosely translates to the smoothness of the predictor function. 
However, it becomes drastically different for classification, due to the discrete nature of class labels. 

The goal of classification is to recover the Bayes optimal decision boundary with the lowest misclassification rate (0-1 loss). 
Decision boundary corresponds to certain level sets of the classifiers, which is more difficult to control than the classifier itself. 
As is often the case, especially in image classification, the classes can be thought of as separable with positive margins, i.e., the class labels have no randomness and images in different classes reside in non-overlapping regions with positive pairwise distances. 
In this case, there are infinitely many possible decision boundaries with zero misclassification error, but only some of them are robust with good generalization properties. 
Current training methods offer little control over the selection process and the resulting decision boundaries often turn out to be unsatisfactory. 
For natural data, it is commonly believed that an ideal decision boundary (e.g., human's), which offers both good accuracy and robustness, should not be too complicated. 
In practice, how to effectively find such decision boundaries can be a real challenge.

Let $\cF$ denote some function space.
In learning theory, the model complexity (how large is $\cF$) is of critical importance, especially for model generalization and robustness \citep{vapnik1999overview, bousquet2002stability, james2013introduction}. 
Certain types of regularization are necessary to prevent over-complication and overfitting of the training data. The same is also true in deep learning, where modern networks are usually overparametrized. 
Various regularization techniques have been developed for training DNNs, e.g., weight decay, dropout \citep{srivastava2014dropout}, batch normalization \citep{ioffe2015batch}, early stopping \citep{prechelt1998early}, etc.
Though their regularization effects are largely implicit, a variety of implicit biases have been recently identified \citep{woodworth2019kernel, chizat2020implicit, razin2020implicit, hu2021regularization, ding2023random}. 
Nevertheless, without exception, all aforementioned types of regularization are on the \textit{functional} level, i.e., regularizing $\cF$ with respect to some complexity measurement. 
However, as we will point out in the next section, the complexity of $\cF$ itself is not of the most interest in classification. Instead, what matters the most are the \textit{level sets} of $\cF$.

\section{Proper Regularization for Classification}
For a function $f:\RR^d\to\mathbb{R}$, let
$\|f\|_\infty=\sup_{\bx\in\RR^d}|f(\bx)|$. 
Let $\PP$ be a probability measure on $\RR^d$ and denote $d_\triangle(G_1, G_2)= \PP(G_1\triangle G_2)= \PP\left((G_1\backslash G_2)\cup (G_2\backslash G_1)\right)$ as the measure of the symmetric difference of sets in $\RR^d$.

Consider the binary classification setting where $\bx\in\RR^d$, $y\in\{-1, 1\}$. 
Let the conditional probability $\eta(\bx)=\PP(y=1|\bx)$. 
Given $\eta(\bx)$, the Bayes optimal decision rule is to assign label $1$ if $\eta(\bx)\ge 1/2$ and label $-1$ if $\eta(\bx) < 1/2$. If the two classes are separated (the supports of two class distributions are disjoint), $\eta$ is a piecewise constant function taking values only from $\{0, 1\}$.
The 0-1 loss is not friendly for optimization \citep{bartlett2006convexity}. Thus, various surrogate losses are employed in practice, e.g., cross-entropy, hinge loss, etc.
In statistics literature, there are two types of assumptions for classification \citep{audibert2007fast}, one on the conditional probability and the other on the decision boundary. 
Classification by estimating the conditional probability is usually referred to as "plug-in" classifiers and it's worth noting that it essentially reduces classification to regression. In comparison, estimating the decision boundary is more fundamental \citep{hastie2009elements}. Hence, characterizing the decision boundary is of critical importance.

\subsection{From Function Space to Level Set}
The goal of classification is to recover the Bayes optimal decision boundary, which divides the input space into non-overlapping regions with respect to labels.  
Therefore, classification is better to be thought of as \emph{estimation of sets} in $\RR^d$, rather than estimation of functions on $\RR^d$.
This is because the set difference reflects the 0-1 loss much more directly than functional norms on $\cF$. To be more specific, if $f\in\cF$ approximates $\eta$ so well that $\|f(\bx)-\eta(\bx)\|_\infty \le 2\epsilon$, there is still no guarantee of matching the sign of $\eta(\bx)-1/2$ close to the decision boundary. Consider a noisy scenario, where the label we observe is flipped relative to the true label with probability $\left(\tfrac{1}{2} - \epsilon\right)$. Then the misclassification rate of $f$ could be arbitrarily bad. 
In contrast, if we have a good estimation of the set $G^* = \{\bx\in \RR^d: \eta(\bx)\ge 1/2\}$ such that $d_\triangle(\hat{G}, G^*)\le \epsilon$, the misclassification probability can be directly bounded by $\epsilon$.

\begin{figure}[t]
    \centering
    \includegraphics[width=0.48\textwidth
    ]{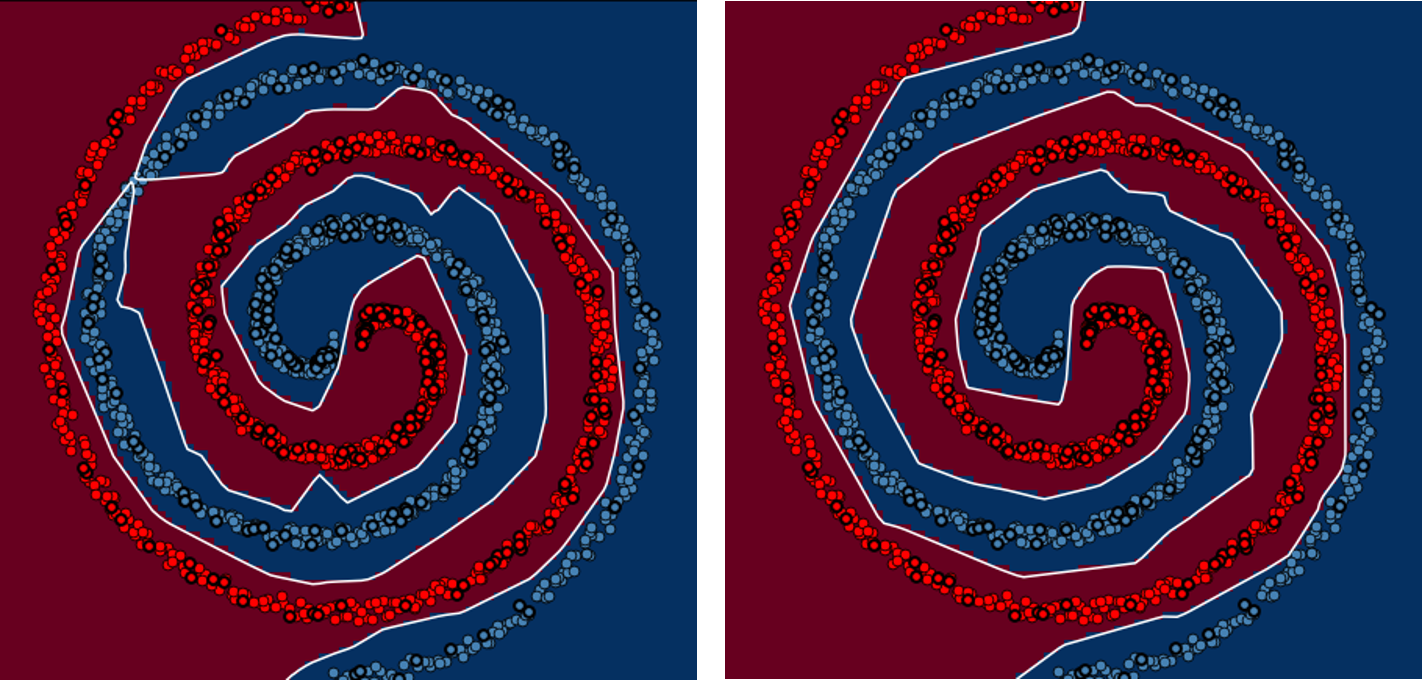}
    \caption{Illustration of a difficult classification task in $[-1,1]^2$ using ReLU classifiers. Two classes (blue and red) are separated. Among all the points, only 300 in each class are training samples, marked with \textbf{thickened} outline. The left figure is from regular training, achieving 99.65\% test accuracy; the right figure is from adversarial training, achieving 100\% test accuracy. The decision boundary on the right is more robust and noticeably less complicated. }
    \label{fig:spiral}
\end{figure}

In practice, the deep classifier is parametrized by a neural network $f\in\cF$ and the decision boundary is its \emph{level set}, $G_f:=\{\bx\in\RR^d: f(\bx)=0\}$, which is modeled \emph{implicitly}. 
Let $\cG = \{G_f: f\in\cF\}$. 
Notice that regularizing $f$ may have no effect on $G_f$ since the level set is invariant to scaling of $f$.
To be more specific, $f(\bx)$ and $\lambda \cdot f(\bx)$ have the same level set, and as $\lambda\to 0$, the majority of commonly used function norms $\|f(\bx)\|$ will tend to zero.
Hence, the complexity of $\cF$ and the complexity of $\cG$ may not be closely connected.

When explicit regularization is absent in training deep classifiers, one may hope the decision boundary complexity is implicitly regularized, either from the model architecture or the training techniques. Unfortunately, this is not supported by empirical evidence in robust transfer learning \citep{shafahi2019adversarially}. Given an adversarially robust teacher model, e.g., from adversarial training, only by vanilla knowledge distillation \citep{hinton2015distilling} and fitting the input-output relationship, the resulting student model, no matter the size, does not retain robustness. 
To achieve comparable robustness, data augmentation on the input space such as mixup samples \citep{muhammad2021mixacm}, or matching intermediate features \citep{goldblum2020adversarially} seems indispensable. 
While matching the classifiers cannot transfer robustness, matching the decision boundary from teacher to student obviously can. 
From this perspective, various data augmentations can be viewed as regularization of the input space, on the decision boundary.   

Adversarial training, noise injection, and margin maximization can all be viewed as means of boundary regularization, pushing decision boundaries away from training samples. 
We show empirically that these methods lead to a significant reduction in boundary complexity, even though their design motivation was different.
Adversarial training can be also viewed as a special form of gradient regularization \citep{lyu2015unified}, or data-dependent operator norm regularization \citep{roth2019adversarial}. Among others, \citet{chan2019jacobian} proposed to directly regularize the saliency of the classifier's Jacobian to improve robustness. 
Adversarial robustness is also shown to improve by replacing the ReLU activation with smooth functions \citep{xie2020smooth}, and modifying the loss function \citep{pang2019rethinking, bao2020calibrated, hu2021understanding}. 
Although the classifier gradient is more related to boundary complexity, these types of regularization methods inspired by adversarial training are not directly targeting the decision boundary. 

In this work, we advocate that for classification, the proper complexity to regularize is the boundary complexity of $\cG$, rather than the functional complexity of $\cF$.
A complexity measurement directly targeting the decision boundary will better reflect classification properties and may be largely independent of known metrics on the function space.

\subsection{Measuring Boundary Complexity}
Now that we have established boundary complexity as the proper, yet missing regularization in classification, the next question is how to measure it.
Compared to functions, boundary complexity measurement is far less explored. 
In statistics literature, classification has been analyzed as a nonparametric estimation of sets problem where the convergence rate critically depends on the complexity of the hypothesis class and the estimator class \citep{mammen1999smooth}.
However, the typical complexity measurements, e.g., bracketing entropy, covering number, Rademacher complexity, etc. are on the group level and cannot evaluate a single set (decision boundary). 
For general classifiers, how to properly quantify the boundary complexity remains an open problem.
\cite{chen2019topological} utilized persistent homology to measure the topological complexity of decision boundaries. \cite{lei2022understanding} characterized boundary complexity by their variability with respect to data and algorithm randomness. \cite{yang2020boundary} proposed the concept of boundary thickness and demonstrated its relationship to classification robustness. However, the aforementioned characterizations of boundary complexity are highly abstract and not explicitly calculable.  

To this end, we consider specifically classifiers with Rectified Linear unit (ReLU) activation, whose decision boundary is piecewise linear, and the boundary complexity can be conveniently characterized by the number of affine pieces, which is intuitive and visually accessible. 
In Figure \ref{fig:spiral}, the left decision boundary has 491 affine pieces while the right one has only 254.
As can be seen in the figure, the less complicated boundary generalizes better and is more robust. 

\begin{remark}[Boundary pieces]
    The count of boundary pieces of ReLU networks might be overly simplified for classification problems, since it does not take the length of each piece and their overall structure into consideration. However, it does offer unique benefits. Besides being intuitive and visually accessible, it also bridges the complexity of the ReLU network itself. It would be interesting to see the relationship between the count of boundary pieces and the total number of linear pieces during training. Other boundary complexities, e.g., boundary thickness, have no counterpart in the function space. 
\end{remark}

For ReLU neural networks, the structure of the affine pieces and, in particular, the number of distinct pieces have been objects of interest.
Sharp bounds (exponential with depth) on the maximum number of affine regions have been investigated \citep{montufar2014number}, demonstrating the benefit of deeper networks. 
\cite{hanin2019complexity} provided a framework to count the number of linear regions of a piecewise linear network.
A method for upper-bounding the number of affine regions \emph{locally} in a ball around a data point was developed in \cite{zhu2020bounding}. Interestingly, both experiments of \cite{zhu2020bounding} on the local number of affine regions and ours on the global count of boundary pieces indicate a two-stage behavior during training.

In classification, we are interested in the boundary pieces (level set) more than in affine regions, and existing literature there is scarce. 
For counting, previous works only compute a \emph{superset} of the decision boundary and therefore give only upper bounds on the exact number (see Proposition 6.1. in \cite{zhang2018tropical} and \cite{alfarra2020decision}).
For linking the count to classification, to the best of the authors' knowledge, the only relevant work is \cite{hu2020sharp}, where a teacher-student classification setting is considered and upper bounds on boundary pieces (bracketing entropy) in ReLU classifiers are utilized to bound the generalization error. Interestingly, \cite{hu2020sharp} showed that when the student network is larger than the teacher, if the boundary complexity is not regularized, the 0-1 loss excess risk convergence rate will not be rate-optimal. 

As we illustrated before, a ReLU network and its level set may share little connection. Calculating the number of boundary pieces is a new and technically challenging problem. Although there might be other ways to characterize the boundary complexity, the boundary piece count does provide a valid starting point for this problem. 

\subsection{Contributions}

In this work, we study the boundary complexity of ReLU classifiers and investigate the number of affine pieces in the decision boundary. The contributions are
\begin{itemize}
    \item 
    With the help of tropical geometry, we provide a novel explicit algorithm for counting the exact number of boundary pieces and affine regions of ReLU networks. In contrast to \cite{zhang2018tropical} and \cite{alfarra2020decision}, we do not require the weights to be integer-valued. Unlike the algorithm of \cite{zhu2020bounding}, which discards some information at each layer, our approach preserves a complete representation of a neural network's functional form.
    
    \item
     We empirically investigate our proposed boundary complexity during training and interesting properties are revealed. 
     First, the boundary piece count is largely independent of other measures during training.
     They (e.g., boundary count, total piece count, and $l_2$ norm of weights) share little similarity during the training process.
     Second, the boundary piece count is negatively correlated with robustness. Adversarial training and noise injection are found to have significant regularizing effects on boundary complexity. 
\end{itemize}

\section{Boundary Complexity of ReLU Networks}

A few works \cite{alfarra2020decision, charisopoulos2018tropical, hertrich2021towards, maragos2021tropical, montufar2021sharp, trimmel2020tropex, zhang2018tropical} on this topic used the ideas of \emph{tropical geometry} - an area of algebraic geometry studying surfaces over the max-plus semi-ring \cite{maclagan2009introduction}. The connection to ReLU networks comes from them being compositions of affine transformations and the rectified linear unit \(\sigma(x) = \max \{0,x\}\). This enables us to write the network as a difference between two convex piece-wise affine functions.
These, in turn, can be interpreted in a useful way in a \emph{dual space}, where affine functions are points and maximum functions correspond to upper convex hulls.
This interpretation allowed \citeauthor{zhang2018tropical} to reprove the best bounds for the largest possible number of affine regions a ReLU network with a given architecture may have.

This section expands on the tropical geometry perspective of ReLU networks. Our main theoretical result is a way to explicitly compute the zero set of a difference of two convex piecewise-affine functions---and therefore compute the exact count of boundary pieces of a ReLU network. 
To improve the readability, we include necessary preliminary results and rephrase them into consistent technical language.
The proofs are mostly omitted and can be found in the appendix.

Let's start with a proposition taken from \cite{magnani2009convex}.
\begin{proposition}\label[proposition]{CPLs}
    A function of the form
    \[f(\x) = \max_{i=1,\ldots,n}\{A_i\x+b_i\}\]
    is convex and piecewise-affine. Also, every convex piecewise-affine function with a finite number of linear pieces is of this form.
\end{proposition}

We will proceed to abbreviate ``convex piecewise-affine'' to CPA and ``difference of convex piecewise-affine'' to DCPA. To be precise, by a ReLU network we mean a neural network where every activation function is the rectified linear unit.

\begin{proposition}\label[proposition]{DCPAs}
    Given any ReLU network, the function defined by it can be written as a DCPA function.
\end{proposition}

Conversely, \cite{ovchinnikov2002max} proved that any piecewise-affine function
with a finite number of linear regions
is a min-max polynomial in its component affine functions.
This implies that it can be written as a DCPA function
and so -- represented by a ReLU network.

\subsection{Tropical Geometry}

In this section, we introduce the aforementioned
interpretation of CPAs in the \emph{dual space} \(\mathbf{D}\).
It may resemble a projective involution,
which makes it even more surprising that notions such as convex hull turn out useful.
We make no distinction between affine functions \(f : \x \mapsto \vec{a}^\intercal \x + b\)
and their graphs \(\{(\x, y) \in \R^{d+1}\ |\ y = f(\x)\}\).
Thus, we identify affine functions \(\R^d \rightarrow \R\)
with hyperplanes in \(\R^{d+1}\)
containing no vertical lines (\(\{\x_0\} \times \R \subseteq \R^{d+1}\)
for some \(\x_0 \in \R^d\));
this ambient \(\R^{d+1}\) will be called the \emph{real space}
and denoted \(\mathbf{R}\).

We make effort to distinguish between \(\mathbf{R}\) and \(\mathbf{D}\)
as both are copies of \(\R^{d+1}\) which may cause confusion.

\begin{definition}
    We say that \((\x, y)\) lies above (the graph of) \(f\) when \(y > f(\x)\). We denote it by \((\x, y) \succ f\).
\end{definition}

\begin{definition}
    For an affine function \(f: \R^d \to \R\) given by \(f(\x) = \vec{a}^\intercal \x + b\),
    we define its \emph{dual} \(\cR^{-1}(f)\)
    as the point \((\vec{a},b)\in \R^{d+1} =: \mathbf{D}\).
    Accordingly, this \(\R^{d+1}\) will be called the \emph{dual space} and denoted \(\mathbf{D}\).
    Conversely, for a dual point \(\vec{c} = (\vec{a}, b) \in \mathbf{D}\),
    we define \(\cR(\vec{c})\) to be the affine function \(\x \mapsto \vec{a}^\intercal \x + b\)
    (i.e. a hyperplane in \(\mathbf{R}\)).
\end{definition}

As we will see from~\cref{duality},
\(\cR\) turns out to interchange the relations of collinearity and concurrence,
extend to planes of any dimensionalities,
preserve orthogonality and sides of hyperplanes.
For consistency, we set: 
\begin{definition}
    To a real point \(\vec{z} = (\x, y) \in \mathbf{R}\),
    we associate as its dual the following hyperplane in \(\mathbf{D}\)
    \[\cR^{-1} (\vec{z}) = (\vec{a} \mapsto (-\x)^\intercal \vec{a} + y).\]
    Conversely, to a dual hyperplane
    \(H = (\vec{a} \mapsto \x^\intercal \vec{a} + y) \subset \mathbf{D}\),
    we associate the real point \[\cR(H) = (-\x, y) \in \mathbf{R}.\]
\end{definition}

Note that the correspondence between dual hyperplanes and real points has an extra sign not present in the pairing of dual points with real planes.

\begin{proposition}\label[proposition]{duality}
    The duality \(\cR\) has the following properties:
    \begin{enumerate}
        \item 
        A dual point \(\vec{c} \in \mathbf{D}\)
        lies on a dual hyperplane \({H \subset \mathbf{D}}\)
        if~and~only~if the corresponding real hyperplane \(\cR(\vec{c}) \subset \mathbf{R}\)
        contains the point \(\cR(H) \in \mathbf{R}\).
        I.e.
        \[\vec{c} \in H \Leftrightarrow \cR(\vec{c}) \ni \cR(H).\]
        \item 
        Points of a dual \(k\)-dimensional plane \(F\) are precisely the duals of real hyperplanes containing some \((d-k)\)-dimensional real plane. We denote this common real \((d-k)\)-dimensional hyperplane as \(\cR(F)\).
        \item 
        Duality is containment-reversing, i.e., \[F \subseteq G \Leftrightarrow \cR(F) \supseteq \cR(G)\] for dual planes \(F, G\), and analogously for \(\cR^{-1}\).
        \item 
        For any real hyperplane \(f\), the projection \(p(\cR^{-1}(f))\) of its dual \(\cR^{-1}(f)\) onto the first \(d\) coordinates is normal to its isolines \(\{\x\ |\ f(\x) = \text{const.}\}\).
        \item 
        Dual point \(\vec{c} \in \mathbf{D}\) lies above the graph of \(H \subset \mathbf{D}\)
        if and only if the real point \(\cR (H) \in \mathbf{R}\)
        lies below the graph of \(\cR(\vec{c}) \subset \mathbf{R}\).
        In symbols
        \[\vec{c} \succ H \Leftrightarrow \cR(\vec{c}) \succ \cR(H).\]
        \item 
        Points \(\vec{c}, \vec{c}'\) that differ only in the \((d+1)\)-th coordinate
        (lie exactly above/below each other)
        correspond precisely to parallel planes (both under \(\cR\) and \(\cR^{-1}\)).
    \end{enumerate}
\end{proposition}

The next proposition shows another property of the duality, crucial to our framework.

\begin{definition}\label[definition]{upper-hull}
    Let \(S \subset \R^{d+1}\) be a finite set of points.
    The convex hull of \(S\) will be denoted \(\mathcal C(S)\).
    Furthermore, we will call
    the set of points
    \[\{(\vec{x}, y) \in \mathcal C(S)\ |\
    (\vec{x}, y + \epsilon) \not\in \mathcal C(S)\ \text{for any}\ \epsilon>0\}\]
    the \emph{upper hull} of \(S\) and denote it 
    \(\U(S)\).
    Finally, the set of vertices of \(\U(S)\)
    will be denoted \(\U^*(S)\).
\end{definition}

\begin{proposition}\label[proposition]{max-hull}
    Let \(S \subset \mathbf{D}\) be a finite set of points.
    Then, for every point \(\x \in \mathbf{D}\) lying below \({\U}(S)\),
    we have (in \(\mathbf{R}\))
    \[\cR(\x) \leq \max\{\cR(\vec{s})\ |\ \vec{s} \in {\U}(S)\},\]
    i.e. the affine function in \(\mathbf{R}\) dual to \(\x\)
    lies fully below the maximum of the affine functions whose duals lie on \({\U}(S)\).
\end{proposition}

\Cref{eg-max-hull} gives us a useful correspondence%
---each CPA function can be represented uniquely
as an upper-convex hull in the dual space.
This allows us to implicitly simplify the notation as well,
as illustrated in \Cref{eg-max-hull}.

\begin{example}\label[example]{eg-max-hull}
    Let us consider the function
    \[f(x) = \max\Big\{-x+3,\: -\tfrac{1}{2}x+2,\: \tfrac{1}{2}x,\: x-2,\: 0\Big\}.\]
    \cref{fig:eg-max-hull} draws it in both the real and dual space.
    \begin{figure}[ht]
        \centering
        \includegraphics[width=0.48\textwidth]{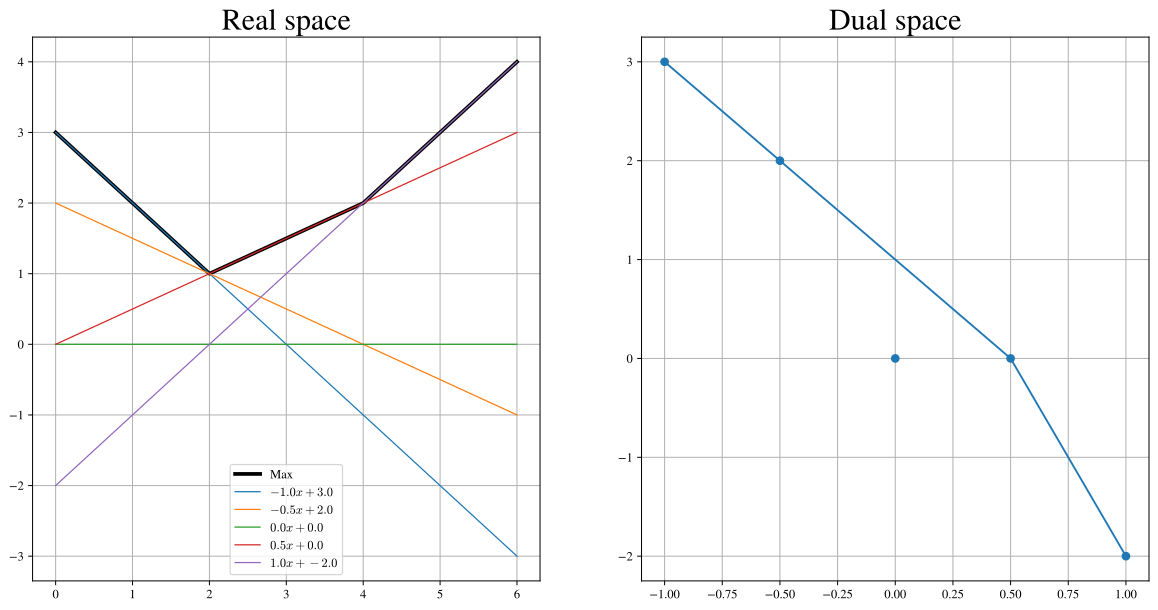}
        \caption{Real and dual diagrams in \cref{eg-max-hull}.}
        \label[figure]{fig:eg-max-hull}
    \end{figure}
    
    We can see that the points \((-\tfrac{1}{2}, 2), (0,0) \in \mathbf{D}\)
    corresponding to the functions \(y=-\tfrac{1}{2}x + 2\) and \(y=0\)
    lie respectively on and under the upper hull of the other points. This means that the functions \(y = -\tfrac{1}{2} x + 2, y=0\) never exceed the maximum of \(-x+3, \tfrac{1}{2}x,\, x-2\), but \(y=-\tfrac{1}{2}x + 2\) matches it at some point.
    
    In particular, we can write the maximum using just three of the functions.
    \begin{align*}
       & \max\Big\{-x+3,\: -\tfrac{1}{2}x+2,\: \tfrac{1}{2}x,\: x-2,\: 0\Big\}\\
       &= \max\Big\{-x+3,\:  \tfrac{1}{2}x,\: x-2\Big\}
    \end{align*}
\end{example}
\medskip

\subsection{ReLU Networks in the Context of Tropical Geometry}

This section shows precisely how to generate the dual diagram of a function defined by a neural network.

Let us denote by \(F_l: \R^{d} \to \R^{w_l}\) the function defined by the network taking the input to the post-activation values on the \(l\)-th layer (here \(w_l\) is the width of the \(l\)-th layer). This means that
\[F_l(\x) = \mathbf{\sigma}(A_l F_{l-1}(\x)).\]

Let us assume that \(F_{l-1} = \cR(P_{l-1}) - \cR(N_{l-1})\) for \(P_{l-1}\) and \(N_{l-1}\) being \emph{vectors} (ordered tuples) of sets of points.
We want to write \(F_l = \cR(P_l) - \cR(N_l)\) for \(P_l\) and \(N_l\) computed in terms of \(P_{l-1}\) and \(N_{l-1}\). For this, we need to introduce some notation.

\begin{definition}
    Given  sets of points \(X, Y \subset \mathbf{D} \cong \R^{d+1}\), we define
    \begin{itemize}
        \item \(X \oplus Y = \{\x + \vec{y}\ |\ \x\in X, \vec{y}\in Y\}\) to be the \emph{Minkowski sum of \(X\) and \(Y\)};
        \item \(X \cup Y\) to be the standard union of \(X\) and \(Y\) as sets.
    \end{itemize}
\end{definition}
We also define these operations on vectors of sets of points to be the coordinate-wise operations. 
These have important interpretations in our correspondence.

In the following, for a finite set \(X \subset \mathbf{D}\)
we identify \(\cR(X)\) with the function \(\max\{\cR(\vec{x})\ |\ \vec{x} \in X\}\)
being a maximum of hyperplanes in \(\mathbf{R}\).

\begin{proposition}\label[proposition]{basic_corr}
    For any sets of points \(X, Y \subset \mathbf{D}\), we have
    \begin{itemize}
        \item \(\cR(X \cup Y) = \max\{\cR(X), \cR(Y)\}\);
        \item \(\cR(X\oplus Y) = \cR(X) + \cR(Y)\).
    \end{itemize}
\end{proposition}

\begin{proof}
    The first one is clear from the definition. For the second one, we have
    \begin{align*}
        &\max\{x_1, \ldots, x_n\} + \max\{y_1, \ldots, y_m\} \\
        &= \max\{x_1+y_1, x_1+y_2, \ldots, x_n+y_m\}.
    \end{align*}
\end{proof}

Now, we need to define matrix multiplication for vectors of sets of points.

\begin{definition}
    Given \(S\subset\mathbf{D}\),
    we define the scalar multiplication \(\lambda\cdot S\) in the usual way.
    For a vector \(X = (X_i)_{1 \le i \le n}\) of sets of points in the dual space
    and for an \(n \times m\) matrix \(A\)
    we define the \emph{Minkowski matrix product of \(X\) by \(A\)} through
    \[(A\otimes X)_i = \bigoplus_{j=1}^{n} A_{ij} \cdot X_j.\]
\end{definition}

Notice that we could run into problems with just using the Minkowski operations,
since as long as \(S\) has at least 2 points,
we will have \(2\cdot S \neq S \oplus S\).
However, if we restrict ourselves to the vertices of upper convex hulls
and non-negative matrices the operations are `well-behaved'.

\begin{proposition}\label[proposition]{basic-properties}
    For matrices \(A,B\) with non-negative values and vectors of points \(X, Y_1, Y_2\),
    the following hold.
    \begin{itemize}
        \item \(\U^*\big((A+B)\otimes X\big) = \U^*\big((A \otimes X) \oplus (B \otimes X)\big)\);
        \item \(A \otimes (Y_1 \oplus Y_2) = (A\otimes Y_1) \oplus (A\otimes Y_2)\);
        \item \(AB\otimes X = A\otimes(B\otimes X)\);
        \item \(X \oplus (Y_1 \cup Y_2) = X \oplus Y_1 \cup X \oplus Y_2\).
    \end{itemize}
\end{proposition}

This seems useful, but quite restrictive, since we need to operate with non-negative matrices. However, every matrix \(A\) can be written as a difference between its positive part and its negative part \(A = A^+ - A^-\), where both \(A^+\) and \(A^-\) are non-negative.

We also have an interpretation for the matrix multiplication, similar to \cref{basic_corr}.
Here, when passing a vector of sets of points to the operator \(\cR\),
we apply it coordinate-wise getting a vector of maximums of affine functions.

\begin{proposition}\label[proposition]{matrix_otimes_points}
    Given a vector \(X\) of sets of points in \(\mathbf{D}\)
    and a non-negative matrix \(A\), we have
    \[A\: \cR(X) = \cR(A\otimes X).\]
\end{proposition}
\begin{proof}
    \[[A\:\cR(X)]_i = \bigoplus_j A_{ij}[\cR(X)]_j = \bigoplus_j [\cR(A_{ij}X_j)] \] \[ = \cR(\oplus_j A_{ij}X_j) = \cR([A\otimes X]_i) = [\cR(A\otimes X)]_i\]
\end{proof}
We can now characterise the function \({F_l = \cR(P_l) - \cR(N_l)}\)
in terms of vectors of points \(P_{l-1}\) and \(N_{l-1}\).
\begin{proposition}\label[proposition]{explicit}
    Let's assume that
    \(F_l = \sigma(A_l\: F_{l-1})\) and \(F_{l-1} = \cR(P_{l-1})-\cR(N_{l-1})\).
    Then, after writing \({A_l = A_l^+ - A_l^-}\),
    we get \(F_l = \cR(P_l) - \cR(N_l)\) for
    \[N_l = (A_l^-\otimes P_{l-1}) \oplus (A_l^+\otimes N_{l-1})\]
    \[\textrm{and} \quad P_l =(A_l^+\otimes P_{l-1})\oplus(A_l^-\otimes N_{l-1}) \:\cup\: N_l.\]
\end{proposition}
Proposition \ref{explicit} is the key to our counting algorithm. Given a neural network, we apply it to all the layers successively, and in the end we obtain a representation of the NN as a DCPA function. Having a DCPA form, we can use proposition \ref{new-hard} and \ref{corr_affine_pc_count} to count the number of boundary and affine pieces.

\subsection{Tropical Hypersurfaces}\label[section]{duals}

In this section, we explore the regions into which a CPA function partitions the plane, which is called the \emph{tessellation} of a CPA. We define it formally below.

\begin{definition}
    Given a CPA
    \[F(\x) = \max\{f_1(\x), \ldots, f_n(\x)\}\]
    where \(f_i\) are affine functions, an \emph{affine region of \(F\)} is
    \[\Big\{\x \in \R^d \: \Big| f_i(\x) = f_{i'}(\x) > f_j(\x) \textrm{ for all } i, i' \in I,  j \in J \Big\},\] 
    where \(I, J\) are disjoint sets whose union is \(\{1, \dots, n\}\).
    Its \emph{dimension} is the smallest dimension
    of an affine subspace of \(\R^{d}\) containing it.
    The set of all regions of dimension \(k\) (\(k\)-cells)
    will be denoted as \(\T_k(F)\), and \(\T(F) = \bigcup_k \T_k(F)\).
\end{definition}
For a set of points \(S\) in the dual space we will denote by \(\T(S)\)
the tessellation of \(\cR(S)\).
For example, \(\T_0\) is the set of all vertices of \(\T(S)\),
\(\T_1\) is the set of all its lines, rays and segments.

\begin{proposition}\label[proposition]{scary}
    \(k\)-cells of \(\T(S)\) are in one-to-one correspondence with \((d-k)\)-cells of \(\U(S)\).
    Each \(k\)-cell \(\sigma\) of \(\T(S)\) is of the form
    \[p(\cR(\text{dual planes tangent to }\U(S)\text{ containing }\sigma')),\]
    where \(\sigma'\) is a \((d-k)\)-cell of \(\U(S)\),
    and \(p : \R^d \times \R \rightarrow \R^d\) is the projection onto first \(d\) coordinates.
\end{proposition}
By \(H\) being \emph{tangent} we mean that the whole of \(\U(S)\) lies \underline{under or on} \(H\)
and that \(H \cap \U(S) \ne \emptyset\).


\subsection{Decision boundary}
 Let \(F = \cR(P)\) and \(G = \cR(N)\) be CPA functions \(\R^d \to \R\). We are interested in being able to describe the zero set \(D\) of a DCPA function \(F-G\). 
The proposition below expands on the idea of Proposition 6.1 in \citep{zhang2018tropical}.

\begin{proposition}\label[proposition]{new-easy}
    Let us assume that no points of \(P\) lie on \(\U(N)\) and vice versa. 
    The set \(D\) is a union of precisely these \({(d-1)}\)-dimensional cells of \({\T(P\cup N)}\) which correspond to the edges of \(\U(P\cup N)\) with one end in \(P\) and the other end in \(N\).
\end{proposition}

This means that to draw the decision boundary, all we have to do is draw the hypersurface \(\T(P\cup N)\) and identify which cells come from the intersection of the graphs of \(\cR(P)\) and \(\cR(N)\).

\Cref{new-easy} deals with the case most likely to happen in general situations, but it is possible that some points of \(P\) lie on \(\U(N)\) or vice versa. \Cref{new-hard} describes this more difficult case too. We compute the boundary count of a neural network by applying \ref{new-hard} to the DCPA representation of a NN (from proposition \ref{explicit}).

\begin{proposition}\label[proposition]{new-hard}
    Let \(F = \cR(P), G = \cR(N)\) be CPA functions. Then the zero set \(D = \{\x \in \R^d \: |\: F(\x) = G(\x)\}\) consists precisely of this cells of \(\T(P \cup N)\), which correspond to the cells of \(\U(P\cup N)\) containing points from both \(P\) and \(N\).
\end{proposition}

\subsection{Affine pieces}

Our formalism also allows us to count the exact total number of affine pieces. To do this for a neural network, we apply the corollary \ref{corr_affine_pc_count} to the DCPA form obtained from proposition \ref{explicit}.

\begin{corollary}\label{corr_affine_pc_count}
    The number of affine pieces (\(d\)-cells) of a DCPA function \(\cR(P) - \cR(N)\) is equal to the number of vertices of \(\U(P \oplus N)\).
\end{corollary}

Corollary~\ref{corr_affine_pc_count} is a special case of a more general result stated below.

\begin{proposition}\label{DCPL_affine_pieces}
    Each \(k\)-cell \(\sigma\) of \(\cR(P) - \cR(N)\) is of the form \[\sigma \! = \! p(\cR(\text{hyperplanes tangent to }\U(P \oplus N)\text{ containing }\sigma'))\] where \(\sigma'\) is a \((d-k)\)-cell of \(\U(P \oplus N)\). The correspondence \(\sigma \leftrightarrow \sigma'\) is bijective.
\end{proposition}
To the best of the authors' knowledge, this explicit formula for counting the total number of affine pieces has not been spelled out in existing literature, where the scaling of the count with respect to neural network structures is usually the focus.

\begin{remark}
    In ReLU neural networks it is possible to have a degenerate situation, where on two regions the network computes the same affine function, but these regions differ in activation patterns. Our approach will see such regions as separate. We do not know of any literature where this would be treated differently.
\end{remark}

\section{Numerical Experiments}
In this section, as a proof of concept, we conduct numerical experiments on 2D synthetic data. The aim of this section is two-fold. 
Firstly, we compare the proposed boundary complexity (\#Boundary) to various other complexity measurements, e.g., the total number of affine pieces (\#Total), the sum of weights squared (F-norm), and evaluate their trends during training. 
The results show that our boundary complexity is quite unique, with distinctive features. 
Secondly, we demonstrate a negative correlation between the number of boundary pieces and classification robustness, where popular robust training methods, specifically noise injection and adversarial training, can both diminish the number of boundary pieces.

We choose ReLU neural networks with 2 hidden layers of different widths across all our simulations.
Three training schemes are considered: regular training with cross-entropy (CE), CE with Gaussian noise injection (Noisy), and CE with $l_\infty$-adversarial training by fast gradient sign attacks \citep{goodfellow2014explaining} (Adv). 
Two synthetic datasets are constructed in 2-dimensional space, one is 3-by-3 Gaussian mixture (Figure \ref{fig:gaussian}) and the other is spiral-shaped (Figure \ref{fig:spiral}). The Gaussian case provides a baseline while the spiral case is much more challenging and may better reflect complicated data structures in practice. 
To measure robustness, we choose Gaussian distributed random noise injection with standard deviation $\sigma$. 2000 test points are used to approximate the expectation and this empirical robustness measure is denoted (in percentile) by $R(\sigma)$.

\begin{figure}[ht]
    \centering
    \includegraphics[width=0.48\textwidth
    ]{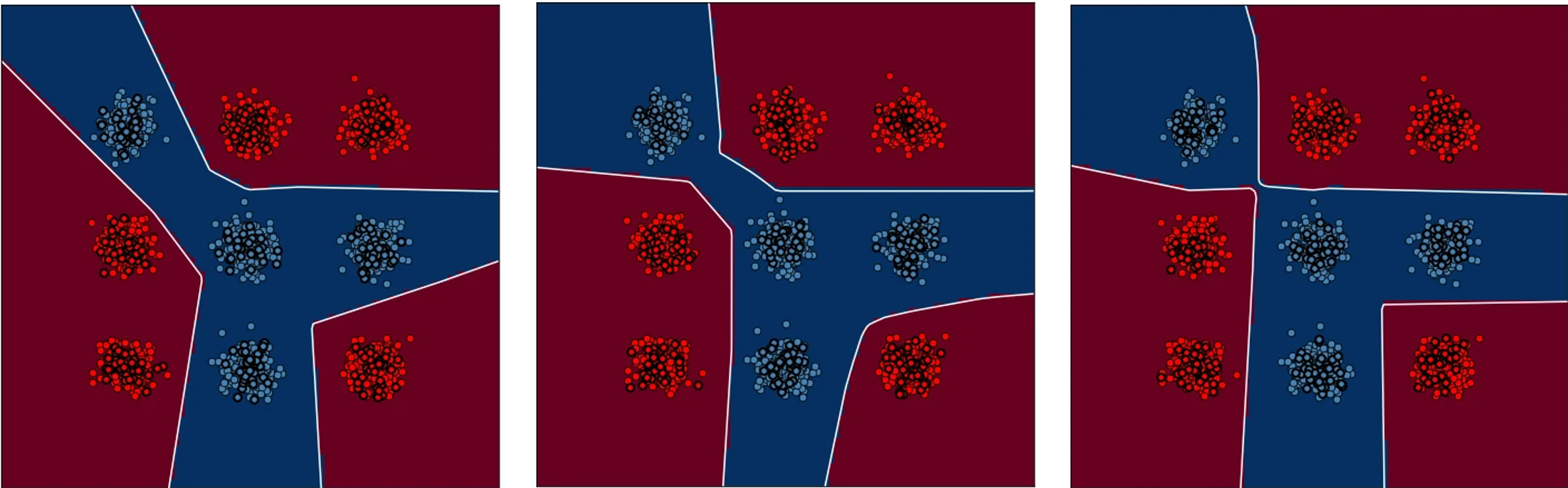}
    \caption{Decision boundaries in the 3$\times 3$ Gaussian mixture case in $[-2, 2]^2$. From left to right are instances of CE (\#Boundary=46), Noisy (\#Boundary=41), Adv (\#Boundary=40), respectively. }
    \label{fig:gaussian}
\end{figure}

The quantities at initialization are shown in Table \ref{tab:robust1} and Table \ref{tab:robust2}. 
We can see that the initial \#Boundary is usually much smaller, with larger variations. This is to be expected as the boundary is only a level set of the initialized classifier, which can be very sensitive to constant shifts.
The initial \#Total is usually larger. This is interesting and indicates that the initial classifier is more random in terms of linear region arrangement.
Like \#Boundary, the F-norm at initialization is much smaller, but with much smaller variations. This is to be expected as the F-norm is directly linked to initialized weights.

\subsection{Trends During Training}
For different tasks, we can observe the overall trend for \#Boundary to be: first increase, then decrease and finally stabilize. 
Similar behaviors can also be observed for \#Total and F-norm during training, but their movements are not synchronized. 
Among the training methods, the overall trends share more similarities than differences, except for with or without weight decay. 
Typical instances are shown in Figure \ref{fig:trend} and \ref{fig:trend2}. 

\begin{figure}[ht]
    \centering
     \includegraphics[width=0.49\textwidth
    ]{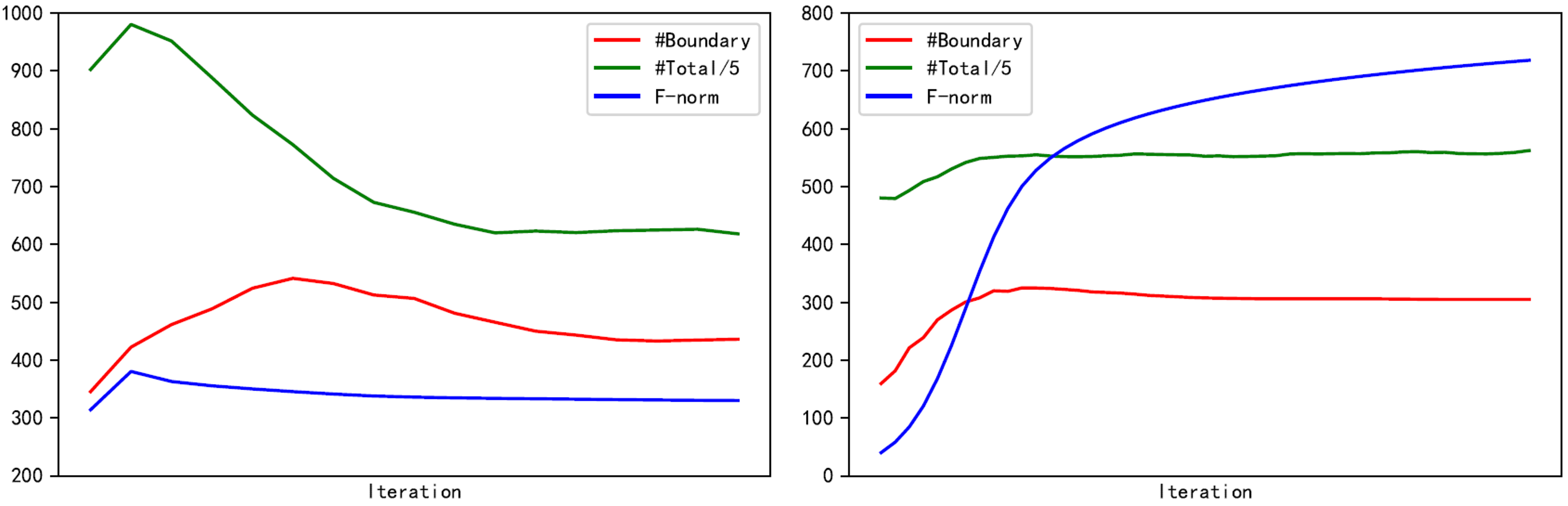}
    \caption{Training trends of \#Boundary (red), \#Total (green), F-norm (red) vs. iteration in the 2D spiral case. Left: CE with weight decay; Right: CE without weight decay. }
    \label{fig:trend}
\end{figure}

\begin{figure}[ht]
    \centering
     \includegraphics[width=0.49\textwidth
    ]{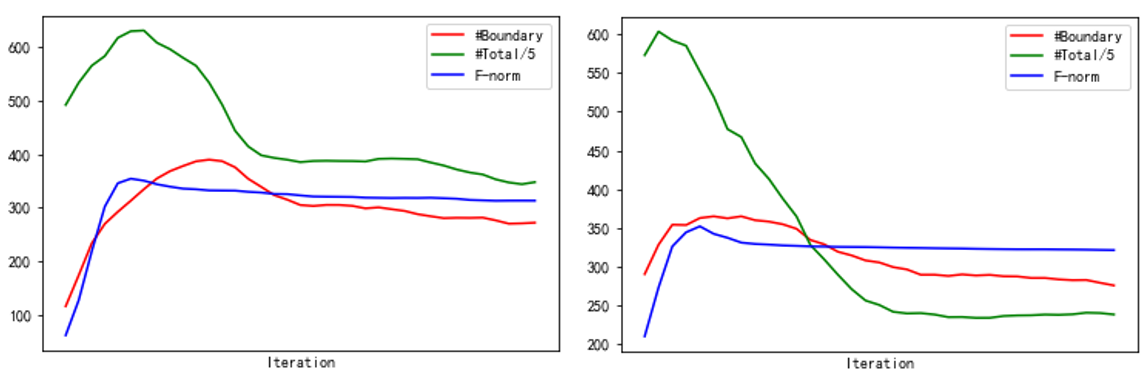}
    \caption{Training trends of \#Boundary (red), \#Total (green), F-norm (red) vs. iteration in the 2D spiral case. Left: Noisy with weight decay; Right: Adv with weight decay. }
    \label{fig:trend2}
\end{figure}

\paragraph{\#Boundary vs others.}
The left figure in Figure \ref{fig:trend} shows the typical trends in the Noisy case with weight decay, where we can clearly see that \#Boundary lags behind the others. 
When the training starts, \#F-norm and \#Total peak much earlier than \#Boundary.  
In most cases, we observe that F-norm peaks first, then \#Total, and lastly \#Boundary. 
When robust training is applied (Noisy, Adv), the gaps among them widen. 
In the later stage, F-norm stabilizes much faster than the others, while we can consistently observe that \#Boundary flattens slower than \#Total. Overall, \#Boundary appears to change much slower than the others, taking more time to peak, and more time to plateau.

\paragraph{Role of weight decay.}
The right figure in Figure \ref{fig:trend} shows a typical trend in the CE case without weight decay, which demonstrates drastically different behaviors. 
\#Boundary and \#Total plateau much earlier and do not change much once the classifier has overfit the training data. In comparison, F-norms keep getting larger, which is to be expected due to the use of cross-entropy loss. 
Weight decay is found to play an important role in the forming of ReLU networks' geometric structures. This is surprising as naively shrinking a ReLU network does not change its affine piece arrangement.

\subsection{Classification Robustness}

In this section, we aim to investigate the relationship between robustness and \#Boundary. However, in the absence of practical algorithms to regularize the boundary complexity, we turn to popular robust training methods and evaluate whether they can significantly reduce \#Boundary. 
Results for the Gaussian mixture and spiral case are reported in Table \ref{tab:robust1} and Table \ref{tab:robust2}, respectively.

\begin{table}
    \caption{Comparison of boundary piece counts in the Gaussian mixture case for ReLU network with layer widths 2-10-10-1. The reported number is an average (standard deviation) of 10 repetitions. 
    }
    \centering
    \scriptsize
    \begin{tabular}{l|ccccc}
    \hline
     & \#Boundary  & \#Total  & F-norm & Acc\% & $R(0.2)$\\
    \hline
    Initial & 29 (17) & 290 (29)   & 6.8 (0.61) & 50.1 (1.1) & -\\ 
    \hline
        CE & 43 (5.3) & 190 (24)   & 57 (3.6) & 100 & 96.4  \\
        Noisy    & 41 (3.1)  & 216 (26)  & 67 (2.6)  & 100 & 97.0  \\
        Adv & 36 (4.6)  & 172  & 73 (2.1) & 100  & 97.2 \\
      \hline
    \end{tabular}
    \label{tab:robust1}
\end{table}
\begin{table}
    \caption{Comparison of boundary piece counts in the 2D spiral case for ReLU network with layer widths 2-30-30-1. The reported number is an average (standard deviation) of 10 repetitions.}
    \centering
    \scriptsize
    \begin{tabular}{l|ccccc}
    \hline
     & \#Boundary  & \#Total  & F-norm & Acc\% & $R(0.02)$\\ \hline
      Initial & 90 (61) & 2432 (179)   & 20 (0.71) & 50.2 (1.2) & -\\ \hline
        CE & 377 (31) & 1915 (207)   & 283 (11) & 93.60 (1.8) & 94.3 (2.2) \\
        Noisy    & 272 (33)  & 1493 (114)  & 322 (17) & 99.15 (0.56) & 98.1 (0.51)\\
        Adv & 259 (21)  & 1241 (135)  & 356 (19) & 99.35 (0.38) & 98.9 (0.36) \\
       \hline
    \end{tabular}
    \label{tab:robust2}
\end{table}

In the simpler Gaussian mixture case, the strength for Noisy and Adv are both set at $0.1$, the same as the variance of each mixing component. 
Figure \ref{fig:gaussian} shows the decision boundaries for CE, Noisy and Adv. Despite the apparent visual difference, the \#Boundary does not differ that much. In Table \ref{tab:robust1}, we can observe \#Boundary to be smaller on average for Noisy and especially Adv. 

The effects of Noisy and Adv become more significant in the harder, more challenging spiral case. CE does not perform as consistently as Noisy or Adv and sometimes will miss the spiral shape.  
The strength for Noisy and Adv are both set at $0.01$, which is roughly the size of the margin. As can be seen from Table \ref{tab:robust2}, both \#Boundary and \#Total significantly dropped while F-norm stays relatively on the same level. 

On both datasets, compared with CE, Noisy and Adv have strong effects on reducing the boundary complexity. The same is not true for function complexity such as F-norm.

\section{Discussion}
We advocate that proper regularization on the decision boundary is of critical importance to classification. As a proof of concept, we choose the number of linear pieces of ReLU networks to measure the boundary complexity, due to its well-definedness. 
The main technical contribution is the explicit formula to count the exact number of boundary pieces as well as total affine pieces. Empirical evaluation and justification are made on synthetic data and interesting properties of the boundary piece count are revealed.

\paragraph{Limitations and extensions.}
(1) While the main focus of this work is on rectified linear units, our method can easily be extended to leaky ReLU activation, and basically all other piecewise linear functions. 
(2) In the experiments, we only evaluated binary classification. However, it is also quite straightforward to count the boundaries between any two given classes in the multi-class classification scenario. 
(3) In the present form, the computation scaling with respect to the network size is impractical for large models, especially with input dimension and depth. The most time-consuming part is the Minkowski sum. However, most of them do not directly contribute to the level set. We believe that further optimizations could shed more light on the mechanics of training procedures. Moreover, incorporating differentiability would give a penalty term that regularizes a previously unaddressed aspect of the network. 
(4) Though intuitive, the number of boundary pieces may not be the best choice for the complexity measurement in classification, since it doesn't take finer details such as piece arrangement into consideration. How to better quantify boundary complexity remains an open question.

\paragraph{Regularizing the boundary complexity.}
Given a measurable boundary complexity, regularizing it during the training process can be challenging.
Adversarial training or noise injection can act as a regularization for boundary complexity, as verified in our experiment.  
Defining suitable boundary complexity measurement and proposing direct and more efficient ways to control it is an open question. 
The aim of this work is to identify such an important problem and convince the readers that boundary complexity is indeed proper to regularize for classification robustness. 
Such regularization is not at odds with other established methods, but a healthy complement to existing literature. 
The level set sampling method proposed in \citet{atzmon2019controlling} may be a good starting point. Uncovering the link of our work to persistent homology \cite{chen2019topological} is also interesting. 
We hope that further work will lead to achieving our ultimate goal -- designing practical and scalable algorithms for effective regularization and thus improving state-of-the-art performance in classification. 


\bibliography{reference}


\newpage
\onecolumn
\section*{Appendix}\label{appendix}

\subsection*{Proof of Proposition \ref{CPLs}}

\begin{proof}
    Denote \(\vec{z} = t \x + (1-t) \y\) and assume that at \(\vec{z}\) the \(i\)-th function is largest, i.e. \(F(\vec{z}) = A_i \vec{z} + b_i\). Then
    \begin{align*}
        F(\vec{z}) = t (A_i \x + b_i) + (1 - t) (A_i \y + b_i) \leq t F(\x) + (1 - t) F(\y)
    \end{align*}
\end{proof}

\subsection*{Proof of Proposition \ref{DCPAs}}

\begin{proof}
    The proof is by induction. We need to prove two facts. First, applying a linear function to a vector of DCPAs produces another vector of DCPAs; second, that a maximum of two DCPAs is a DCPA.
    
    Let \(F-G\) be a vector of DCPAs, where \(F\) and \(G\) are vectors of \(n\) CPAs and \(A\) be an \(m\times n\) matrix with real coefficients. Write \(A = A_+ - A_-\) where both \(A_+\) and \(A_-\) have non-negative entries. Then we have
    \[A\:(F-G) = (A_+ - A_-)\:(F-G) = (A_+F + A_-G) - (A_-F + A_+G).\]
    This proves the first fact.
    
    The second fact is easy to see from \(\max\{a,b\}+c = \max\{a+c,b+c\}\) and \(\max\{a,\max\{b,c\}\} = \max\{a,b,c\}\).
\end{proof}

\subsection*{Proof of Proposition \ref{duality}}

\begin{proof}
The proof follows the following steps. 
    \begin{enumerate}
        \item Let \(\vec{c} = (\vec{a}, b)\) and \(H = (\vec{a} \mapsto \x^\intercal \vec{a}^\intercal + y\). Then both \(\vec{c} \in H\) and \(\cR(H) \in \cR(\vec{c})\) are equivalent to \(b = \x^\intercal \vec{a} + y\).
        \item \(k\)-dimensional dual plane \(F\) can be written as an intersection of \(d-k\) dual hyperplanes \(\cR^{-1}(\vec{z}_0), \dots, \cR^{-1}(\vec{z}_{d-k})\). A dual point \(\cR^{-1}(f)\) belongs to \(F\) if and only if it is a dual of a real hyperplane \(f\) that contains the real points \(\vec{z}_0, \dots, \vec{z}_{d-k}\). Their affine span is the common plane we are looking for, and what we christen \(\cR(F)\).
        
        It is affinely spanned by \(d-k+1\) points, so its dimension is at most \(d-k\). If it was smaller, we could forget some \(\vec{z}_i\), which means that \(F\) was an intersection of \(d-k-1\) hyperplanes, and had dimension at least \(k+1\).
        \item \(F\) is contained in \(G\) if and only if for any hyperplane \(H\) we have \[G \subseteq H \Rightarrow F \subseteq H\] This happens precisely when for all points \(\vec{z} = \cR(H)\) we have \[\vec{z} \in \cR(G) \Rightarrow \vec{z} \in \cR(F)\] that is \(\cR(G) \subseteq \cR(F)\).
        \item Let \(f : \x \mapsto \vec{a}^\intercal \x + b\). Then \(p(\cR^{-1}(f)) = \vec{a}\), which is perpendicular to surfaces \(\vec{a}^\intercal \x = \text{const.}\)
        \item Let \(\vec{c} = (\vec{a}, b)\) and \(H : \vec{d} \mapsto \x^\intercal \vec{d} + y\). Then both \(\vec{c} \succ H\) and \(\cR(\vec{c}) \succ \cR(H)\) are equivalent to \(b > \vec{x}^\intercal \vec{a} + y\).
        \item Suppose \(\vec{c} = (\vec{a}, b),  \vec{c}' = (\vec{a}, b + \Delta)\), and denote \(f : \x \mapsto \vec{a}^\intercal \x + b\). Then \(\cR(\vec{c}) = f, \cR(\vec{c}') = f + \Delta\) -- these functions differ by a constant, so specify parallel planes. The proof for \(\cR^{-1}\) is analogous.
    \end{enumerate}
\end{proof}

\subsection*{Proof of Proposition \ref{max-hull}}

\begin{proof}
    Firstly, let us compare the planes dual to two points, \(\vec{s}_1\) and \(\vec{s}_2\), such that \(\vec{s}_1\) lies directly above \(\vec{s}_2\). This means that they differ only at the very last coordinate---let's say that \(\vec{s}_1 = (\vec{a}_1,b_1)\) and \(\vec{s}_2 = (\vec{a}_2,b_2)\) where \(b_1 \ge b_2\). Then the dual planes \(\cR(\vec{s}_1)\) and \(\cR(\vec{s}_2)\) are precisely
    \[\cR(\vec{s}_1) = \{(\vec{x},y_1) | y_1 = (\vec{a}_1)^\intercal \x + b_1\}, \] 
    \[\cR(\vec{s}_2) = \{(\vec{x},y_2) | y_2 = (\vec{a}_2)^\intercal \x + b_2\},\]
    and since \((\vec{a}_1)^\intercal\x + b_1 \geq (\vec{a}_2)^\intercal\x + b_2\) for all \(\x \in \R^d\), the plane \(\cR(\vec{s}_1)\) lies above \(\cR(\vec{s}_2)\).
    
    Secondly, let us consider a point \(\vec{s}\) in the dual space lying on a segment whose endpoints are \(\vec{s}_1\) and \(\vec{s}_2\). But then for some \(p \in [0,1]\) we have \(\vec{s} = p\cdot \vec{s}_1 + (1-p)\cdot \vec{s}_2\) and thus
    \[(\vec{s})^\intercal \begin{bmatrix} \x \\ 1\end{bmatrix} = p\cdot \Bigg((\vec{s}_1)^\intercal \begin{bmatrix} \x \\1 \end{bmatrix}\Bigg) + (1-p)\cdot \Bigg((\vec{s}_2)^\intercal \begin{bmatrix} \x \\ 1 \end{bmatrix}\Bigg),\]
    so, in particular,
    \[(\vec{s})^\intercal \begin{bmatrix} \x \\ 1\end{bmatrix} \leq \max \Bigg\{(\vec{s}_1)^\intercal \begin{bmatrix} \x \\1 \end{bmatrix},\quad (\vec{s}_2)^\intercal \begin{bmatrix} \x \\ 1 \end{bmatrix}\Bigg\}.\]
    
    Thirdly, we want to piece the two together. For a point \(\vec{s}_2\) lying below \(\U(S)\), let us choose a point \(\vec{s}_1\in \U(S)\) lying exactly above \(\vec{s}_2\). The plane defined by it lies above the one defined by \(\vec{s}_2\) according to the first paragraph. Now we only need to show that points on \(\U(S)\) define planes lying below the minimum, but this follows from the second paragraph and the fact that all points on a convex hull of a finite set of points can be generated by taking segments whose ends lie in the hull and adding all of the points of the segment to the hull.
\end{proof}

\subsection*{Proof of Proposition \ref{basic-properties}}

\begin{proof}
    This is a straightforward consequence of the more elementary identities for scalar \(a, b\): after reducing to upper hulls we have
    \begin{align}
        (a + b) X =& (a X) \oplus (b X) \label{apdx:scalar_right_distributivity}\\
        a (X \oplus Y) =& (a X) \oplus (a Y) \\
        (a b) X =& a (b X) \\
        a (X \cup Y) =& (a X) \cup (a Y)
    \end{align}
    Except~\ref{apdx:scalar_right_distributivity}, all of these hold even before taking the hull. To deal with this one, note that
    \begin{align*}
        (a + b) X &= \{ ax + bx | x \in X \} \\
        &\subseteq \{ ax_1 + bx_2 | x_1, x_2 \in X \} = (aX) \oplus (bX)
    \end{align*}
    so we have \(\U((a + b) X) \subseteq \U \big( (a X) \oplus (b X) \big)\). To see the reverse inclusion, write
    \begin{equation*}
        a x_1 + b x_2 = \tfrac{a}{a+b} (a+b) x_1 + \tfrac{b}{a+b} (a+b) x_2
    \end{equation*}
    which means that
    \begin{equation*}
        (a X) \oplus (b X) \subseteq \U \big( (a+b) X \big)
    \end{equation*}
\end{proof}

\subsection*{Proof of Proposition \ref{explicit}}

\begin{proof}
    Firstly, let us note that
    \begin{align*}
        A_l\: F_{l-1} =& (A_l^+ - A_l^-)\: \big(\cR(P_{l-1}) - \cR(N_{l-1})\big) \\
     =& \big(A_l^+\: \cR(P_{l-1}) + A_l^-\: \cR(N_{l-1})\big) \\
     & - \big(A_l^-\: \cR(P_{l-1}) + A_l^+\: \cR(N_{l-1})\big) \\
     =& \cR\big((A_l^+\otimes P_{l-1}) \oplus (A_l^- \otimes N_{l-1})\big) \\
     &- \cR\big((A_l^-\otimes P_{l-1}) \oplus (A_l^+ \otimes N_{l-1})\big).
    \end{align*}
    Now, we use the fact that \(\max\{x-y,0\} = \max\{x,y\}-y\) to get that for \(N_l = (A_l^-\otimes P_{l-1}) \oplus (A_l^+ \otimes N_{l-1})\), we have
    \begin{align*}
        &\sigma(A_l\: F_{l-1}) \\
        &= \max\{\cR\big((A_l^+\otimes P_{l-1}) \oplus (A_l^- \otimes N_{l-1})\big), \cR(N_l)\} - \cR(N_l)\\
        &=\cR\big((A_l^+\otimes P_{l-1}) \oplus (A_l^- \otimes N_{l-1}) \cup N_l\big) - \cR(N_l),
    \end{align*}
    and thus, for \(P_l = (A_l^+\otimes P_{l-1}) \oplus (A_l^- \otimes N_{l-1}) \cup N_l\), we get
    \[F_l = \sigma(A_l F_{l-1}) = \cR(P_l) - \cR(N_l).\]
\end{proof}

\subsection*{Proof of Proposition \ref{scary}}

\begin{proof}
    \(k\)-cell of \(\T(S)\) is the region defined by the system
    \begin{align}
        f_{i_0} (\x) =& \dots = f_{i_{d-k}} (\x) \label{decision_boundary_linear}\\
        f_{i_0} (\x) \geq& f_j(\x) \text{ for } j \neq i_0, \dots, i_{d-k} \nonumber
    \end{align}
    This can be written as \[(\x, y) \in f_{i_0}, \dots, f_{i_{d-k}} \qquad (\x, y) \succcurlyeq f_j\] In dual space this becomes 
    \[\cR^{-1}((\x, y)) \ni \cR^{-1}(f_{i_0}), \dots, \cR^{-1}(f_{i_{d-k}}) \]
    \[\cR^{-1}((\x, y)) \succcurlyeq \cR^{-1}(f_j)\] Therefore, the duals of points of the \(k\)-cell are precisely the dual planes containing the \((d-k)\)-cell on vertices \(\cR^{-1}(f_{i_0}), \dots, \cR^{-1} (f_{i_{d-k}})\) and tangent to the upper convex hull.
\end{proof}

\subsection*{Proof of Proposition \ref{new-easy}}

\begin{proof}
    The cell of \(\T(P \cup N)\) is a boundary cell iff in the equation~\ref{decision_boundary_linear}, we have both some function \(f_i \in \cR(P)\) and some function \(g_j \in \cR(N)\). This happens exactly when the dual cell has some vertex \(\cR^{-1}(f_i) \in P\) as well as some vertex \(\cR^{-1}(g_j) \in N\).
\end{proof}

\subsection*{Proof of Proposition \ref{new-hard}}

\begin{proof}
    Again, as before, we need to identify those linear pieces of \(\max\{F,G\}\), which lie on the linear pieces of \(F\) and of \(G\). However, this means identifying cells of \(\U(P \cup N)\) which \emph{contain} a cell of \(\U(P)\) and a cell of \(\U(N)\) (this is due to the duality reversing containment of hyperplanes; we mean set-wise containment here, not containment as subcells).
\end{proof}

\subsection*{Proof of Proposition \ref{DCPL_affine_pieces}}

\begin{proof}
    A \(k\)-dimensional cell \(\sigma\) is the set of \(\x\) satisfying the system
    \begin{align*}
        f_{i_0} (\x) &= \dots = f_{i_a} (\x) = s > f_{i'} (\x) \\
        g_{j_0} (\x) &= \dots = g_{j_b} (\x) = t > g_{j'} (\x)
    \end{align*}
    Where \(a + b = d - k\). This can be expressed as relations in the real space
    \begin{align*}
        (\x, s) &\in f_{i_0}, \dots, f_{i_a} \qquad (\x, s) \succ f_{i'} \\
        (\x, t) &\in g_{j_0}, \dots, g_{j_b} \qquad (\x, t) \succ g_{j'}
    \end{align*}
    After passing to the dual space this becomes
    \begin{align}
        \cR^{-1} \big( (\x, s) \big) &\ni \cR^{-1} (f_{i_0}), \dots, \cR^{-1} (f_{i_a}) \label{eqn:pts_on_face_f} \\
        \cR^{-1} \big( (\x, s) \big) &\succ \cR^{-1} (f_{i'}) \label{eqn:pts_below_face_f} \\
        \cR^{-1} \big( (\x, t) \big) &\ni \cR^{-1} (g_{j_0}), \dots, \cR^{-1} (g_{j_b}) \label{eqn:pts_on_face_g} \\
        \cR^{-1} \big( (\x, t) \big) &\succ \cR^{-1} (g_{j'}) \label{eqn:pts_below_face_g}
    \end{align}
    We know that \(\cR^{-1} \big( (\x, s) \big) \) and \(\cR^{-1} \big( (\x, t) \big) \) are a pair of parallel hyperplanes; the former is tangent to \(\U(P)\) (\ref{eqn:pts_below_face_f}) and contains its \(a\)-cell (\ref{eqn:pts_on_face_f}), while the latter is tangent to \(\U(N)\) (\ref{eqn:pts_below_face_g}) and contains its \(b\)-cell (\ref{eqn:pts_on_face_g}).
    
    View these hyperplanes as subsets of \(\mathbb{R}^{d+1}\) and consider their Minkowski sum \(\cR^{-1} \big( (\x, s) \big) \oplus \cR^{-1} \big( (\x, t) \big)\). It is straightforward to verify that it equals the hyperplane \(\cR^{-1} \big( (\x, s + t) \big)\). Since the relation \(\succ\) of lying above is preserved by translations, we have
    \begin{equation*}
        \cR^{-1} \big( (\x, s + t) \big) = \cR^{-1} \big( (\x, s) \big) \oplus \cR^{-1} \big( (\x, t) \big) \succcurlyeq \cR^{-1} (f_i) + \cR^{-1} (g_j) \qquad\qquad \text{for all } \cR^{-1} (f_i) \in P, \cR^{-1} (g_j) \in N
    \end{equation*}
    This means that the plane \(\cR^{-1} \big( (\x, s + t) \big)\) is tangent to \(\U(P \oplus N)\). Also, it contains the \((a + b = d - k)\)-cell \(\sigma'\) on vertices
    \begin{equation}\label{eqn:cell_in_minkowski_sum_PN}
        \left\{ \cR^{-1}(f_{i_\alpha}) + \cR^{-1}(g_{j_\beta}) \ | \ 0 \leq \alpha \leq a, 0 \leq \beta \leq b \right\}
    \end{equation}

    Conversely, suppose a hyperplane \(H\) is tangent to \(\U(P \oplus Q)\) and contains the \((d-k)\)-cell \(\sigma'\)on the vertices from equation \ref{eqn:cell_in_minkowski_sum_PN}. Let \(\x = p(H)\) be the vector of linear coefficients of \(H\). If we had \(f_{i'} (\x) > f_{i_\alpha} (\x)\) for any \(i' \notin \{ i_0, \dots, i_a \} \ni i_\alpha\), then the point \(\cR^{-1} (f_{i'}) + \cR^{-1} (g_{j_0})\) would lie above \(H\), which is impossible. Therefore we must have
    \begin{equation}
        f_{i_0} (\x) = \dots = f_{i_a} (\x) > f_{i'} (\x)
    \end{equation}
    and a similar set of conditions involving \(g\)'s. This means that \(x = p(H)\) lies in the real cell \(\sigma\).

    These functions are mutually inverse, and hence provide a bijection between real points of \(\sigma\) and dual tangent hyperplanes containing \(\sigma'\).

    Since every point of the real space belongs to a unique cell, and every dual hyperplane tangent to \(\U (P \oplus N)\) intersects it in a unique cell, the assignment \(\sigma \leftrightarrow \sigma'\) is bijective.
\end{proof}
\begin{remark}
    Sign of the function on the cell (equivalently,  the class to which the region belongs) depends on which of \(\cR^{-1} \big( (\x, s) \big), \cR^{-1} \big( (\x, t) \big)\) lies above the other.
\end{remark}

\subsection*{Numerical Experiments Details} \label{app:exp}
The neural networks are initialized by the default Uniform distribution\footnote{The default weight initialization in \texttt{torch.nn.linear} is uniform on $[-\sqrt{1/N}, \sqrt{1/N}]$ where $N$ is the width.}. 
For all ReLU neural networks, the optimization is done by stochastic gradient descent with learning rate$=0.1$, momentum$=0.9$ and weight decay$=0.001$ (if not specified otherwise). 

\paragraph{2D spiral}
The synthetic spiral data is from the two-dimensional distribution $P = (\rho \sin \theta + 0.04,  \rho \cos \theta)$ where $\rho = {(\theta/4 \pi)}^{4/5} + \epsilon$ with selected $\theta$ from $(0, 4\pi]$ and $\epsilon \sim \text{unif}([-0.03,0.03])$. We draw 300 positive and 300 negative training samples from $-P$ and $P$, respectively, with a random seed fixed for every run. Both the Gaussian noise injection strength and the adversarial training strength are set at $0.01$. 

\paragraph{2D Gaussian mixture}
There are $3\times 3$ mixing components, each is an isotropic Gaussian with standard deviation $\sigma=0.1$. The means are grid points from $\{-1, 0, 1\} \times \{-1, 0, 1\}$. The mixing weight is equal for all components.
Both the Gaussian noise injection strength and the adversarial training strength are set at $0.1$.

Below we show some training trends for CE, Noisy and Adv in the Gaussian mixture case. It is worth noting that all trend plots in this work, including Figure \ref{fig:trend} and \ref{fig:trend2} are smoothed with moving averages. 

\begin{figure}[ht]
    \centering
     \includegraphics[width=0.5\textwidth
    ]{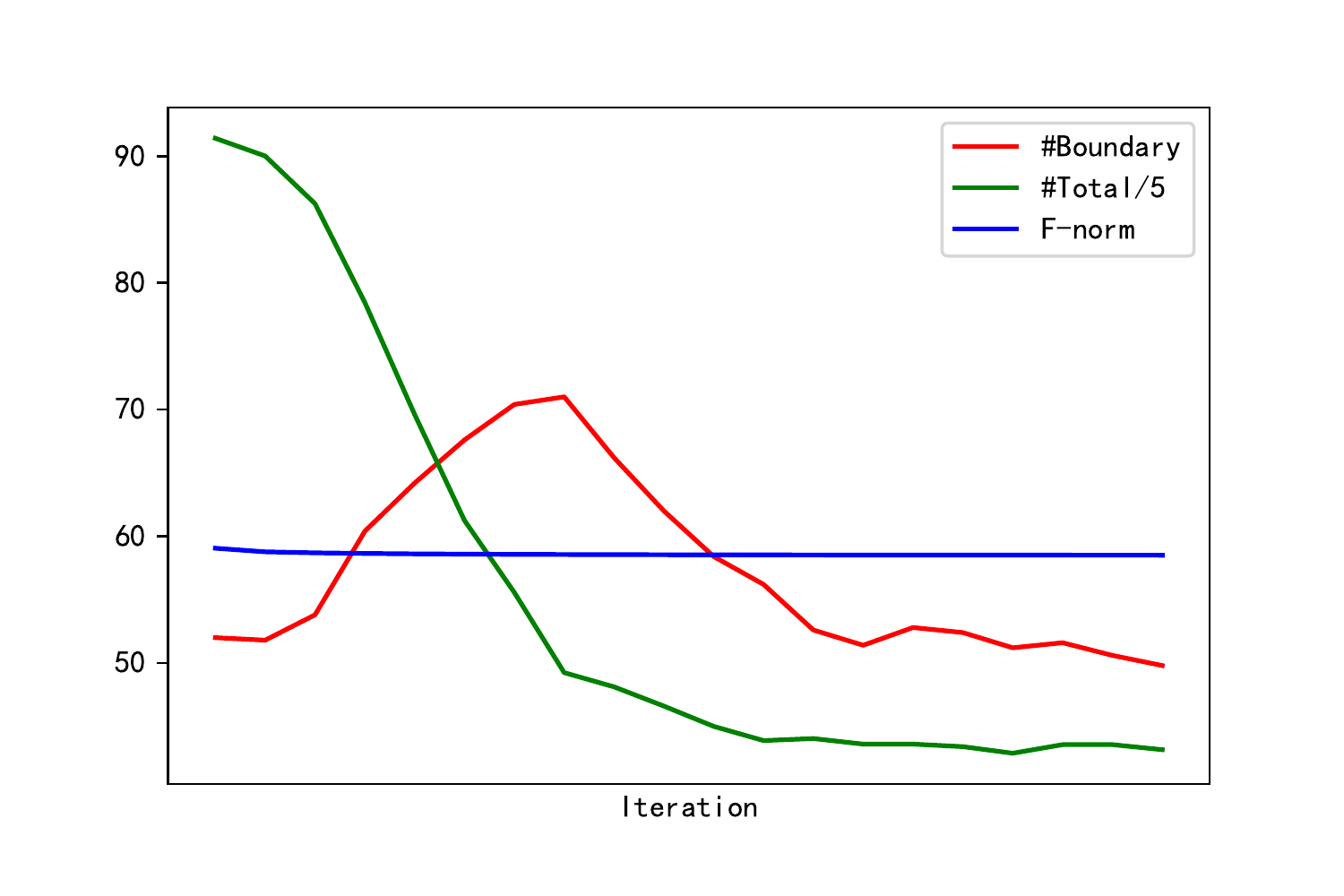}
    
    \caption{CE training trends of \#Boundary (red), \#Total (green), F-norm (red) vs. iteration in the Gaussian mixture case.}
  
\end{figure}

\begin{figure}[ht]
    \centering
     \includegraphics[width=0.5\textwidth
    ]{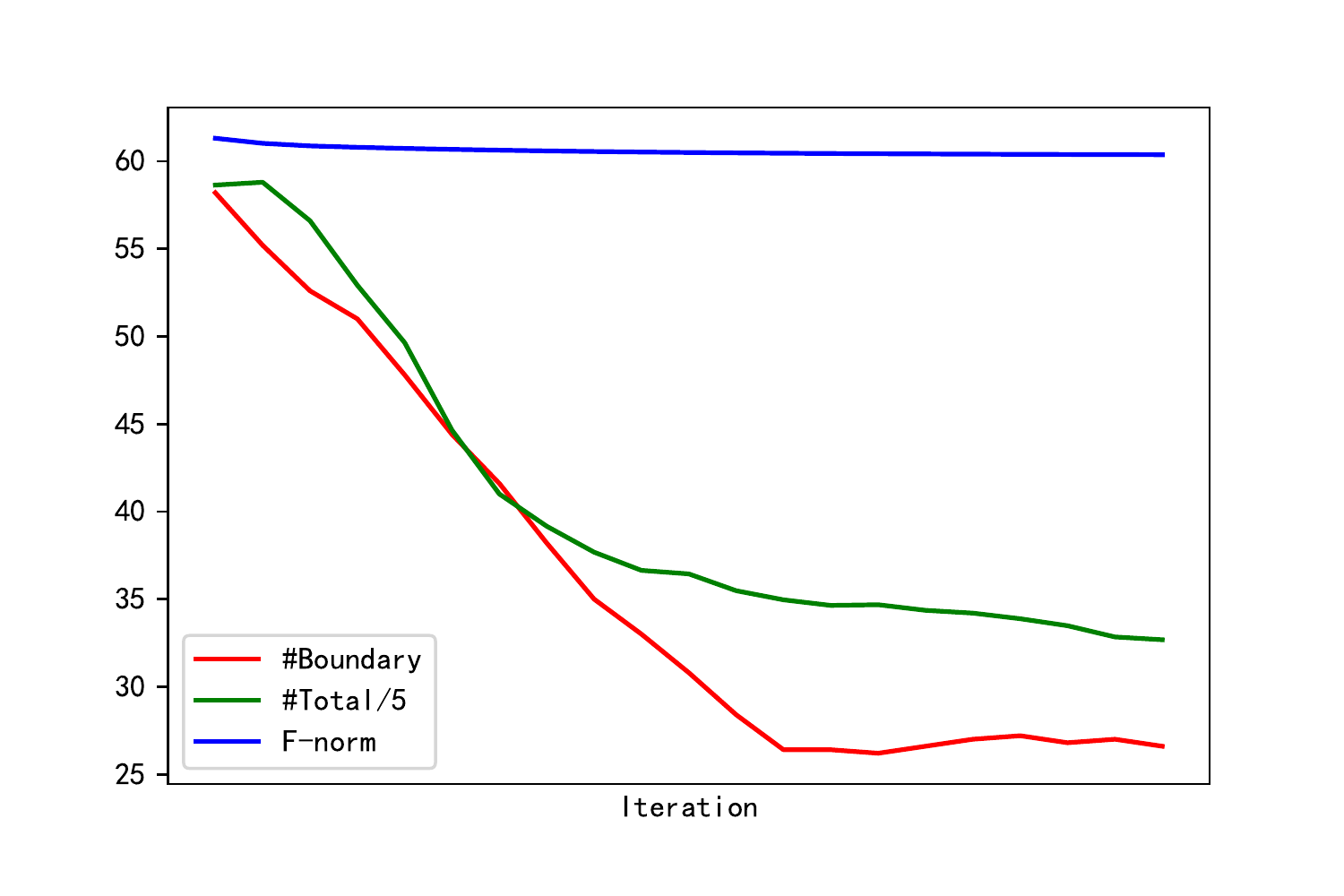}
   
    \caption{Noisy training trends of \#Boundary (red), \#Total (green), F-norm (red) vs. iteration in the Gaussian mixture case.}
    
\end{figure}

\begin{figure}[!ht]
    \centering
     \includegraphics[width=0.5\textwidth
    ]{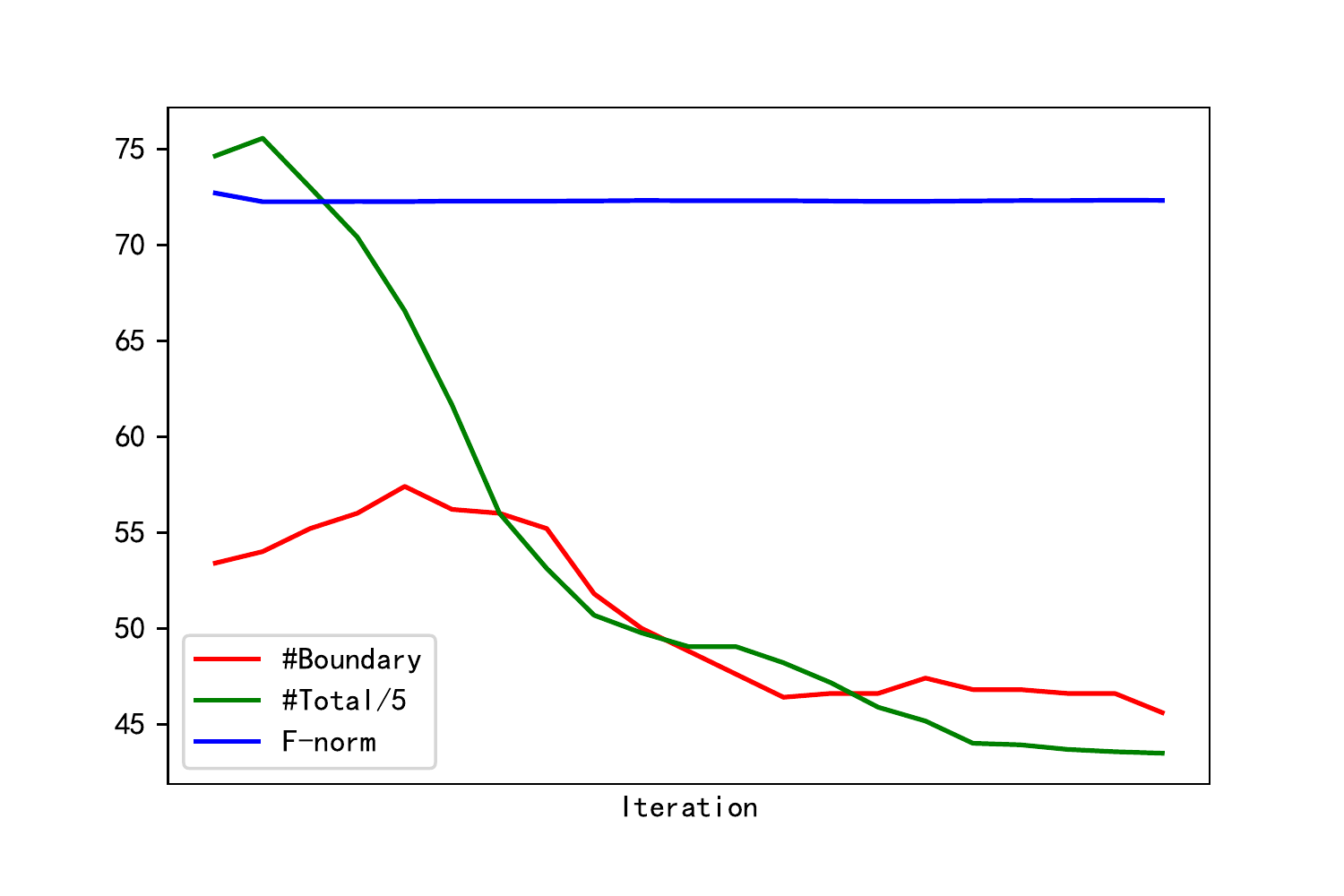}
   
    \caption{Adv training trends of \#Boundary (red), \#Total (green), F-norm (red) vs. iteration in the Gaussian mixture case.}
    
\end{figure}

\subsection*{An example of computation with Proposition~\ref{explicit}}

First we should note that
in a standard ReLU network the transition functions
are any affine functions
but we can introduce a `dummmy dimension' to realise these
as \emph{linear} functions.

We will consider a very simple network
with two-dimensional input,
one hidden layer with three neurons,
and the following transition matrices
(with the dummy dimension included).
For illustrative purposes we assume that
ReLU is applied also at the last layer.

\begin{equation*}
    A_1 = \begin{bmatrix}
        1   & -0.5  & 4  \\
        -2  & 1     & 0  \\
        3   & 3     & -1 \\
        0   & 0     & 1
    \end{bmatrix},\quad
    A_2 = \begin{bmatrix}
        0.5 & -1    & -0.5 & 2 \\
        0   & 0     & 0    & 1
    \end{bmatrix}
\end{equation*}

The input function \(F_0 = (x, y, 1)\)
(where the last coordinate is a dummy)
is decomposed into \(\cR(P_0) - \cR(N_0)\)
with
\begin{equation*}
    P_0 = \begin{pmatrix}
        \{(1,0,0)\} \\
        \{(0,1,0)\} \\
        \{(0,0,1)\}
    \end{pmatrix},\quad
    N_0 = \begin{pmatrix}
        \{(0,0,0)\} \\
        \{(0,0,0)\} \\
        \{(0,0,0)\} \\
    \end{pmatrix}.
\end{equation*}

To compute \(P_1\) and \(N_1\),
we need to decompose the matrix \(A_1\)
into its positive and negative parts \(A_1^+\) and \(A_1^-\).
\begin{align*}
    N_1 & = (A_1^+ \otimes N_0) \oplus (A_1^- \otimes P_0) \\
        & =
        \left(\begin{bmatrix}
            1   & 0     & 4  \\
            0   & 1     & 0  \\
            3   & 3     & 0  \\
            0   & 0     & 1
        \end{bmatrix}
        \otimes
        \begin{pmatrix}
            \{(0,0,0)\} \\
            \{(0,0,0)\} \\
            \{(0,0,0)\} \\
        \end{pmatrix}
        \right)\oplus\left(
        \begin{bmatrix}
            0   & 0.5   & 0  \\
            2   & 0     & 0  \\
            0   & 0     & 1 \\
            0   & 0     & 0
        \end{bmatrix}
        \otimes
        \begin{pmatrix}
            \{(1,0,0)\} \\
            \{(0,1,0)\} \\
            \{(0,0,1)\}
        \end{pmatrix}\right)       \\
        & =
        \begin{pmatrix}
            1 \{(0,0,0)\} \oplus 0 \{(0,0,0)\}
            \oplus 4 \{(0,0,0)\} \\
            0 \{(0,0,0)\} \oplus 1 \{(0,0,0)\}
            \oplus 0 \{(0,0,0)\} \\
            3 \{(0,0,0)\} \oplus 3 \{(0,0,0)\}
            \oplus 0 \{(0,0,0)\} \\
            0 \{(0,0,0)\} \oplus 0 \{(0,0,0)\}
            \oplus 1 \{(0,0,0)\} \\
        \end{pmatrix}
        \oplus
        \begin{pmatrix}
            0 \{(1,0,0)\} \oplus 0.5 \{(0,1,0)\}
            \oplus 0 \{(0,0,1)\} \\
            2 \{(1,0,0)\} \oplus 0 \{(0,1,0)\}
            \oplus 0 \{(0,0,1)\} \\
            0 \{(1,0,0)\} \oplus 0 \{(0,1,0)\}
            \oplus 1 \{(0,0,1)\} \\
            0 \{(1,0,0)\} \oplus 0 \{(0,1,0)\}
            \oplus 0 \{(0,0,1)\} \\
        \end{pmatrix}               \\
        & =
        \begin{pmatrix}
            \{(0,0,0)\} \\
            \{(0,0,0)\} \\
            \{(0,0,0)\} \\
            \{(0,0,0)\} \\
        \end{pmatrix}
        \oplus
        \begin{pmatrix}
            \{(0,0.5,0)\} \\
            \{(2,0,0)\} \\
            \{(0,0,1)\} \\
            \{(0,0,0)\} \\
        \end{pmatrix}
        =
        \begin{pmatrix}
            \{(0,0.5,0)\} \\
            \{(2,0,0)\} \\
            \{(0,0,1)\} \\
            \{(0,0,0)\} \\
        \end{pmatrix}           \\
    P_1 & = (A_1^+ \otimes P_0) \oplus (A_1^- \otimes N_0) \cup N_1 \\
    & =
    \begin{pmatrix}
        \{(1,0,4)\} \\
        \{(0,1,0)\} \\
        \{(3,3,0)\} \\
        \{(0,0,1)\} \\
    \end{pmatrix}
    \oplus
    \begin{pmatrix}
        \{(0,0,0)\} \\
        \{(0,0,0)\} \\
        \{(0,0,0)\} \\
        \{(0,0,0)\} \\
    \end{pmatrix}
    \cup
    \begin{pmatrix}
        \{(0,0.5,0)\} \\
        \{(2,0,0)\} \\
        \{(0,0,1)\} \\
        \{(0,0,0)\} \\
    \end{pmatrix}           \\
    & =
    \begin{pmatrix}
        \{(1,0,4)\} \\
        \{(0,1,0)\} \\
        \{(3,3,1)\} \\
        \{(0,0,1)\} \\
    \end{pmatrix}
    \cup
    \begin{pmatrix}
        \{(0,0.5,0)\} \\
        \{(2,0,0)\} \\
        \{(0,0,1)\} \\
        \{(0,0,0)\} \\
    \end{pmatrix}
    =
    \begin{pmatrix}
        \{(1,0,4), (0, 0.5, 0)\} \\
        \{(0,1,0), (2, 0, 0)\} \\
        \{(3,3,1), (0, 0, 1)\} \\
        \{(0,0,0), (0,0,1)\} \\
    \end{pmatrix}
    \substack{=}_{\U^*}
    \begin{pmatrix}
        \{(1,0,4), (0, 0.5, 0)\} \\
        \{(0,1,0), (2, 0, 0)\} \\
        \{(3,3,1), (0, 0, 1)\} \\
        \{(0,0,1)\} \\
    \end{pmatrix}
\end{align*}
The last operation
is reducing to the upper hull vertices
and it doesn't change the dual function \(\cR(P_1)\).

We repeat this calculation for the next layer.

\begin{align*}
    N_2 & = (A_2^+ \otimes N_1) \oplus (A_2^- \otimes P_1)    \\
    & =
    \left(
    \begin{bmatrix}
        0.5 & 0    & 0 & 2 \\
        0   & 0     & 0    & 1
    \end{bmatrix}
    \otimes
    \begin{pmatrix}
        \{(0,0.5,0)\} \\
        \{(2,0,0)\} \\
        \{(0,0,1)\} \\
        \{(0,0,0)\} \\
    \end{pmatrix}
    \right)
    \oplus
    \left(
    \begin{bmatrix}
        0 & 1    & 0.5 & 0 \\
        0   & 0     & 0    & 0
    \end{bmatrix}
    \otimes
    \begin{pmatrix}
        \{(1,0,4), (0, 0.5, 0)\} \\
        \{(0,1,0), (2, 0, 0)\} \\
        \{(3,3,1), (0, 0, 1)\} \\
        \{(0,0,1)\} \\
    \end{pmatrix}
    \right)                      \\
    & =
    \begin{pmatrix}
        0.5\{(0,0.5,0)\} \\
        \{(0,0,0)\} \\
    \end{pmatrix}
    \oplus
    \begin{pmatrix}
        1\{(0,1,0), (2, 0, 0)\} \oplus 0.5\{(3,3,1), (0, 0, 1)\} \\
        \{(0,0,0)\} \\
    \end{pmatrix}                        \\
    & =
    \begin{pmatrix}
        \{(0,0.25,0)\} \\
        \{(0,0,0)\} \\
    \end{pmatrix}
    \oplus
    \begin{pmatrix}
        \{(0,1,0), (2, 0, 0)\} \oplus \{(1.5,1.5,0.5), (0, 0, 0.5)\} \\
        \{(0,0,0)\} \\
    \end{pmatrix}                         \\
    & =
    \begin{pmatrix}
        \{(0,0.25,0)\} \\
        \{(0,0,0)\} \\
    \end{pmatrix}
    \oplus
    \begin{pmatrix}
        \{(1.5, 2.5, 0.5), (0,1,0.5), (3.5, 1.5, 0.5), (2, 0, 0.5)\} \\
        \{(0,0,0)\} \\
    \end{pmatrix}                                   \\
    & =
    \begin{pmatrix}
        \{(1.5, 2.75, 0.5), (0,1.25,0.5), (3.5, 1.75, 0.5), (2, 0.25, 0.5)\} \\
        \{(0,0,0)\} \\
    \end{pmatrix} \\
    P_2 & = (A_2^+ \otimes P_1) \oplus (A_2^- \otimes N_1) \cup N_2 \\
    & =
    \left(
    \begin{bmatrix}
        0.5 & 0    & 0 & 2 \\
        0   & 0     & 0    & 1
    \end{bmatrix}
    \otimes
    \begin{pmatrix}
        \{(1,0,4), (0, 0.5, 0)\} \\
        \{(0,1,0), (2, 0, 0)\} \\
        \{(3,3,1), (0, 0, 1)\} \\
        \{(0,0,1)\} \\
    \end{pmatrix}
    \right)
    \oplus
    \left(
    \begin{bmatrix}
        0 & 1    & 0.5 & 0 \\
        0   & 0     & 0    & 0
    \end{bmatrix}
    \otimes
    \begin{pmatrix}
        \{(0,0.5,0)\} \\
        \{(2,0,0)\} \\
        \{(0,0,1)\} \\
        \{(0,0,0)\} \\
    \end{pmatrix}
    \right)
    \\&\quad\cup
    \begin{pmatrix}
        \{(1.5, 2.75, 0.5), (0,1.25,0.5), (3.5, 1.75, 0.5), (2, 0.25, 0.5)\} \\
        \{(0,0,0)\} \\
    \end{pmatrix} \\
    & =
    \begin{pmatrix}
        0.5\{(1,0,4), (0, 0.5, 0)\} \oplus 2\{(0,0,1)\} \\
        1 \{(0,0,1)\}
    \end{pmatrix}
    \oplus
    \begin{pmatrix}
        1\{(2,0,0)\} \oplus 0.5\{(0,0,1)\} \\
        \{(0,0,0)\}
    \end{pmatrix}
    \\&\quad\cup
    \begin{pmatrix}
        \{(1.5, 2.75, 0.5), (0,1.25,0.5), (3.5, 1.75, 0.5), (2, 0.25, 0.5)\} \\
        \{(0,0,0)\} \\
    \end{pmatrix} \\
    & =
    \begin{pmatrix}
        \{(0.5,0,2), (0, 0.25, 0)\} \oplus \{(0,0,2)\} \\
        \{(0,0,1)\}
    \end{pmatrix}
    \oplus
    \begin{pmatrix}
        \{(2,0,0)\} \oplus \{(0,0,0.5)\} \\
        \{(0,0,0)\}
    \end{pmatrix}
    \\&\quad\cup
    \begin{pmatrix}
        \{(1.5, 2.75, 0.5), (0,1.25,0.5), (3.5, 1.75, 0.5), (2, 0.25, 0.5)\} \\
        \{(0,0,0)\} \\
    \end{pmatrix} \\
    & =
    \begin{pmatrix}
        \{(0.5,0,4), (0, 0.25, 2)\} \\
        \{(0,0,1)\}
    \end{pmatrix}
    \oplus
    \begin{pmatrix}
        \{(2,0,0.5)\} \\
        \{(0,0,0)\}
    \end{pmatrix}
    \\&\quad\cup
    \begin{pmatrix}
        \{(1.5, 2.75, 0.5), (0,1.25,0.5), (3.5, 1.75, 0.5), (2, 0.25, 0.5)\} \\
        \{(0,0,0)\} \\
    \end{pmatrix} \\
    & =
    \begin{pmatrix}
        \{(2.5,0,4.5), (2, 0.25, 2.5)\} \\
        \{(0,0,1)\}
    \end{pmatrix}
    \\&\quad\cup
    \begin{pmatrix}
        \{(1.5, 2.75, 0.5), (0,1.25,0.5), (3.5, 1.75, 0.5), (2, 0.25, 0.5)\} \\
        \{(0,0,0)\} \\
    \end{pmatrix} \\
    & =
    \begin{pmatrix}
        \{(2.5,0,4.5), (2, 0.25, 2.5), (1.5, 2.75, 0.5), (0,1.25,0.5), (3.5, 1.75, 0.5), (2, 0.25, 0.5)\} \\
        \{(0,0,1), (0,0,0)\} \\
    \end{pmatrix}
\end{align*}

Now, to reduce the result to the upper hull vertices,
we can note that
\[\frac{1}{5}(0,1.25,0.5)
+ \frac{4}{5}(2.5,0,4.5)
= (0,0.25, 0.1)  + (2, 0, 3.6) = (2, 0.25, 3.7) \succ (2, 0.25, 2.5), (2, 0.25, 0.5),\]
so the two points of the right hand side can be dropped without changing the upper hull.
This gives
\begin{equation*}
    P_2 \substack{=}_{\U^*}
    \begin{pmatrix}
        \{(2.5,0,4.5), (1.5, 2.75, 0.5), (0,1.25,0.5), (3.5, 1.75, 0.5)\} \\
        \{(0,0,1)\}
    \end{pmatrix}.
\end{equation*}

Finally, let's recover the representation as a DCPA function.

\begin{align*}
    F_2(x, y) & = \left(\cR(P_2) - \cR(N_2)\right)(x, y)  \\
    & = \max\{1.5x + 2.75y + 0.5, 1.25y + 0.5, 3.5x + 1.75y + 0.5, 2.5x+4.5\} \\
    & - \max\{1.5x+2.75y+0.5, 1.25y + 0.5, 3.5x + 1.75y + 0.5, 2x + 0.25y + 0.5\}
\end{align*}

\subsection*{Examples of application of propositions \ref{new-hard} and corollary \ref{corr_affine_pc_count}}\label{apdx:DCPA_counting_examples}

\paragraph{One-dimensional example} Consider
\begin{align*}
    & f_1 (x) = -\tfrac{1}{2} x - \tfrac{3}{2} \qquad & f_2 (x) =& \tfrac{1}{2} x + \tfrac{1}{2} \qquad & f_3 (x) = 2 x + 1 \\
    & g_1 (x) = 0 \qquad & g_2 (x) =& 2 x \qquad & g_3 (x) = 3 x - 1
\end{align*}

The DCPA function \(F (x) = \max \{ f_1(x), f_2(x), f_3(x) \} - \max \{ g_1(x), g_2(x), g_3(x) \}\) is plotted in figure \ref{subfig:1d_example_plot}. It has 5 affine regions and 3 zeros.

It is represented by dual points as
\begin{align*}
    \max \{ f_1, f_2, f_3 \} = \mathcal{R} (P), \qquad P = \left\{ \begin{pmatrix} -\tfrac{1}{2} \\ -\tfrac{3}{2} \end{pmatrix}, \begin{pmatrix} \tfrac{1}{2} \\ \tfrac{1}{2} \end{pmatrix}, \begin{pmatrix} 2 \\ 1 \end{pmatrix} \right\} \\
    \max \{ g_1, g_2, g_3 \} = \mathcal{R} (N), \qquad N = \left\{ \begin{pmatrix} 0 \\ 0 \end{pmatrix}, \begin{pmatrix} 2 \\ 0 \end{pmatrix}, \begin{pmatrix} 3 \\ -1 \end{pmatrix} \right\}
\end{align*}
Their upper convex hull \(\U (P \cup N)\) is shown on figure \ref{subfig:1d_example_dual_points}. As predicted by proposition \ref{new-hard}, the zero set of \(F\) is in bijection with 1-cells of \(\U (P \cup N)\) which join a point of \(P\) with a point of \(N\). This bijection is shown explicitly in table \ref{subfig:1d_example_cells_to_zeros}. The \(x\)-coordinates of zeros of \(F\) are given by negative slopes of these 1-cells.

The hull of the Minkowski sum \(P \oplus N\) is shown in figure \ref{subfig:1d_example_dual_sum}. In agreement with corollary \ref{corr_affine_pc_count}, there are 5 vertices on \(\U (P \cup N)\). The explicit bijections between the vertices of \(\U (P \cup N)\) and affine regions of \(F\), and between tangents at each vertex and points of the corresponding linear region, is given in the table \ref{subfig:1d_example_tangents_to_points}.

\paragraph{Two-dimensional example} Take
\begin{align*}
    & f_1 = - x + y + 4 & \qquad & f_2 = x + y -2 & \qquad f_3 = - 2 x - y - 1 \\
    & g_1 = 0 & \qquad & g_2 = 2 x - y + 2 & \qquad g_3 = - x + 2 y + 2
\end{align*}
which correspond to dual points
\begin{equation*}
    P = \left\{ \begin{pmatrix} -1 \\ 1 \\ 4 \end{pmatrix}, \begin{pmatrix} 1 \\ 1 \\ -2 \end{pmatrix}, \begin{pmatrix} -2 \\ -1 \\ -1 \end{pmatrix} \right\}, \qquad N = \left\{ \begin{pmatrix} 0 \\ 0 \\ 0 \end{pmatrix}, \begin{pmatrix} 2 \\ -1 \\ 2 \end{pmatrix}, \begin{pmatrix} -1 \\ 2 \\ 2 \end{pmatrix} \right\}
\end{equation*}
The function \(F = \max \{ f_1, f_2, f_3 \} - \max \{ g_1, g_2, g_3 \}\) is shown on figure \ref{subfig:2d_example_plot}. There are 7 affine regions and 6 boundary pieces.

The configuration of dual points \(P \cup N\) is shown on figure \ref{subfig:2d_example_dual_points}. The upper convex hull \(\U (P \cup N)\) contains 4 faces, 8 edges and 5 vertices. As predicted by proposition \ref{new-hard}, edges joining a point of \(P\) with a point of \(N\) correspond precisely to those affine regions of \(F\) which contain a boundary piece. Explicitly, these are \(f_1 - g_2, f_1 - g_3, f_2 - g_2, f_2 - g_3, f_3 - g_2, f_3 - g_3\).

The Minkowski sum \(P \oplus N\) is shown in figure \ref{subfig:2d_example_dual_sum}. In agreement with corollary \ref{corr_affine_pc_count}, 7 of the vertices lie on the upper convex hull. Explicitly, the functions \(f_1 - g_1\) and \(f_1 - g_2\) are the only ones which do not have a nonempty affine region, and the points \(\cR^{-1} (f_1) + \cR^{-1} (g_1)\) and \(\cR^{-1} (f_2) + \cR^{-1} (g_2)\) are the only ones which lie fully below the upper convex hull.

\begin{figure}
    \centering
    \subfigure[Plot of \(F\).]{
        \begin{tikzpicture}[scale=1.5]
        \draw[->] (-3.5, 0) -- (4.5, 0);
        \draw[->] (0, -1.5) -- (0, 1.5);
        \draw (-3, -.1) -- (-3, .1) (-2, -.1) -- (-2, .1) (-1, -.1) -- (-1, .1) (1, -.1) -- (1, .1) (2, -.1) -- (2, .1) (3, -.1) -- (3, .1) (4, -.1) -- (4, .1);
        \draw (-.1, -1) -- (.1, -1) (-.1, 1) -- (.1, 1);
        \draw (-3.5, .25) -- (-2, -.5) -- (-1/3, 1/3) -- (0, 1) -- (1, 1) -- (3, -1);
    \end{tikzpicture}
    \label{subfig:1d_example_plot}
    }
    \subfigure[Points of \(P\) (marked \(\circ\)) and \(N\) (marked \(\times\)) in the dual space. Dashed lines are the upper convex hull \(\U (P \cup N)\). Double lines join a point of \(P\) with a point of \(N\).]{
        \begin{tikzpicture}
        \draw[->] (-1.5, 0) -- (3.5, 0);
        \draw[->] (0, -2.5) -- (0, 1.5);
        \draw (-1, -.1) -- (-1, .1) (1, -.1) -- (1, .1) (2, -.1) -- (2, .1) (3, -.1) -- (3, .1);
        \draw (-.1, -2) -- (.1, -2) (-.1, -1) -- (.1, -1) (-.1, 1) -- (.1, 1);
        \node at (-.5, -1.5) {$\circ$}; \node at (-.8, -1.2) {\small $f_1$};
        \node at (.5, .5) {$\circ$}; \node at (.4, .8) {\small $f_2$};
        \node at (2, 1) {$\circ$}; \node at (2.4, 1) {\small $f_3$};
        \node at (0, 0) {$\times$}; \node at (.3, -.3) {\small $g_1$};
        \node at (2, 0) {$\times$}; \node at (1.8, -.3) {\small $g_2$};
        \node at (3, -1) {$\times$}; \node at (2.7, -1.2) {\small $g_3$};
        \draw[dashed] (-.5, -1.5) (.5, .5) -- (2, 1);
        \draw[dashed, double] (-.5, -1.5) -- (0, 0) -- (.5, .5) (2, 1) -- (3, -1);
    \end{tikzpicture}
    \label{subfig:1d_example_dual_points}
    }
    \hspace{8mm}
    \subfigure[Correspondence between 1-cells of \(U (P \cup N)\) and zeros of \(F\). We represent 1-cell as a graph of a linear function over an interval.]{
        \begin{tabular}{c|c|c}
        vertices & 1-cell & zero of \(F\) \\
        \hline
        \(f_1, g_1\) & \(3x\) over \([-\tfrac{1}{2}, 0]\) & -3 \\
        \(f_2, g_2\) & \(x\) over \([0, \tfrac{1}{2}]\) & -1 \\
        \(f_3, g_3\) & \(-2x+4\) over \([2, 3]\) & 2
    \end{tabular}
    \label{subfig:1d_example_cells_to_zeros}
    }
    \subfigure[Minkowski sum \(P \oplus N\) in the dual space. Dashed lines represent the upper convex hull.]{
        \begin{tikzpicture}[scale=.8]
        \draw[->] (-1.5, 0) -- (5.5, 0);
        \draw[->] (0, -3) -- (0, 1.5);
        \draw (-1, -.1) -- (-1, .1) (1, -.1) -- (1, .1) (2, -.1) -- (2, .1) (3, -.1) -- (3, .1) (4, -.1) -- (4, .1) (5, -.1) -- (5, .1);
        \draw (-.1, -3) -- (.1, -3) (-.1, -2) -- (.1, -2) (-.1, -1) -- (.1, -1) (-.1, 1) -- (.1, 1);
        \filldraw (-.5, -1.5) circle (1.5pt); \node at (-1, -1) {\small $f_1 - g_1$};
        \filldraw (1.5, -1.5) circle (1.5pt);
        \filldraw (2.5, -2.5) circle (1.5pt);
        \filldraw (.5, .5) circle (1.5pt); \node at (.2, .8) {\small $f_2 - g_1$};
        \filldraw (2.5, .5) circle (1.5pt);
        \filldraw (3.5, -.5) circle (1.5pt);
        \filldraw (2, 1) circle (1.5pt); \node at (2, 1.3) {\small $f_3 - g_1$};
        \filldraw (4, 1) circle (1.5pt); \node at (4.4, 1.3) {\small $f_3 - g_2$};
        \filldraw (5, 0) circle (1.5pt); \node at (5, -.3) {\small $f_3 - g_3$};
        \draw[dashed] (-.5, -1.5) -- (.5, .5) -- (2, 1) -- (4, 1) -- (5, 0);
    \end{tikzpicture}
    \label{subfig:1d_example_dual_sum}
    }
    \hspace{5mm}
    \subfigure[Tangents to \(U (P \cup N)\) and the corresponding affine regions. We represent each tangent line as a graph of a linear function. They are parameterised by their slope \(t\).]{
        \begin{tabular}{c|c|c}
            vertex & tangents & affine region \\
            \hline
            \( f_1 - g_1 \) & \(\{t x + \tfrac{t-3}{2}\}_{t \in [2, \infty)}\) & \((-\infty, -2]\) \\
            \( f_2 - g_1 \) & \(\{t x + \tfrac{1 - t}{2}\}_{t \in [\tfrac{1}{2}, 2]}\) & \([-2, -\tfrac{1}{2}]\) \\
            \( f_3 - g_1\) & \(\{t x + 1 - 2 t\}_{t \in [0, \tfrac{1}{2}]}\)& \([-\tfrac{1}{2}, 0]\) \\
            \(f_3 - g_2\) & \(\{t x + 1 - 4 t\}_{t \in [-1, 0]}\) & \([0, 1]\) \\
            \(f_3 - g_3\) & \(\{t x - 5 t\}_{t \in (-\infty, -1]}\) & \([1, \infty)\)
        \end{tabular}
        \label{subfig:1d_example_tangents_to_points}
    }
    \caption{Illustration of the results on a one-dimensional example.}
\end{figure}

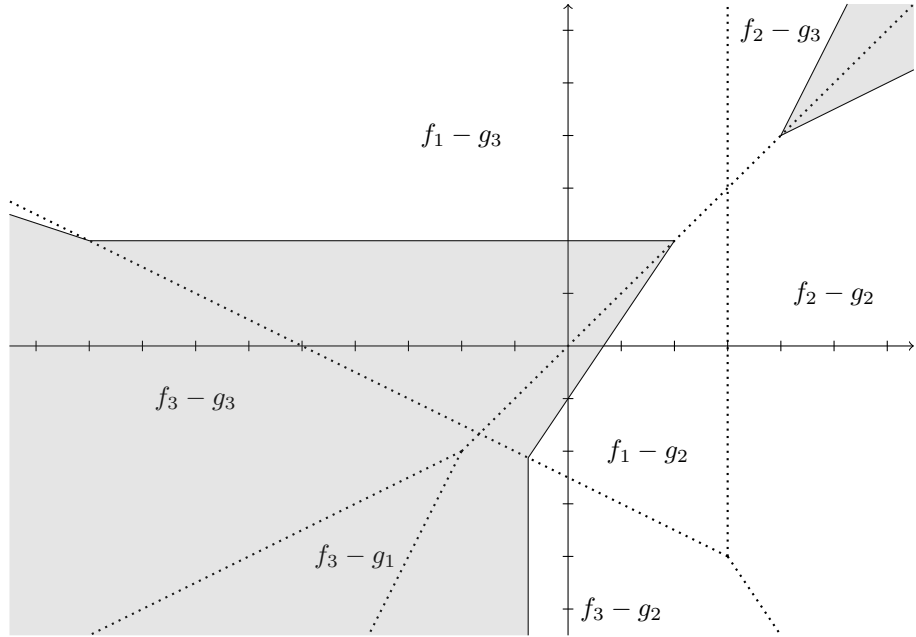
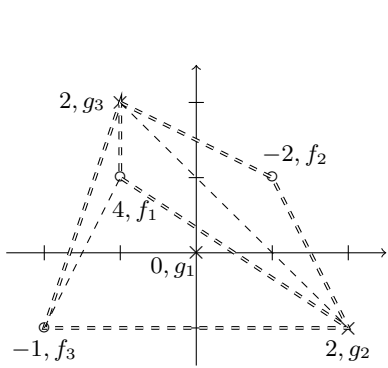
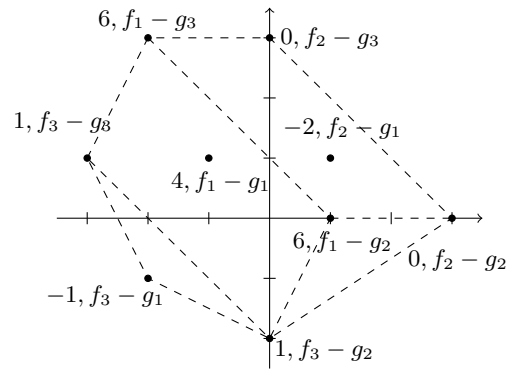
\begin{figure}[t]
    \centering
    \subfigure[The function \(F\) on the \(xy\) plane. Dotted lines mark the boundaries of affine regions. The annotation \(f_i - g_j\) means that \(F = f_i - g_j\) on the corresponding affine region. Solid lines indicate the zero set, and the shaded region contains arguments for which \(F\) is positive.]{
    \begin{tikzpicture}[scale=.7]
        \draw[->] (-10.5, 0) -- (6.5, 0);
        \draw[->] (0, -5.5) -- (0, 6.5);
        \draw (-10, -.1) -- (-10, .1) (-9, -.1) -- (-9, .1) (-8, -.1) -- (-8, .1) (-7, -.1) -- (-7, .1) (-6, -.1) -- (-6, .1) (-5, -.1) -- (-5, .1) (-4, -.1) -- (-4, .1) (-3, -.1) -- (-3, .1) (-2, -.1) -- (-2, .1) (-1, -.1) -- (-1, .1) (1, -.1) -- (1, .1) (2, -.1) -- (2, .1) (3, -.1) -- (3, .1) (4, -.1) -- (4, .1) (5, -.1) -- (5, .1) (6, -.1) -- (6, .1);
        \draw (-.1, -5) -- (.1, -5) (-.1, -4) -- (.1, -4) (-.1, -3) -- (.1, -3) (-.1, -2) -- (.1, -2) (-.1, -1) -- (.1, -1) (-.1, 1) -- (.1, 1) (-.1, 3) -- (.1, 3) (-.1, 4) -- (.1, 4) (-.1, 5) -- (.1, 5) (-.1, 6) -- (.1, 6);
        \draw[dotted, thick] (-2, -2) -- (-9, -5.5) (-2, -2) -- (-3.75, -5.5) (-2, -2) -- (6.5, 6.5);
        \draw[dotted, thick] (3, -4) -- (3, 6.5) (3, -4) -- (-10.5, 2.75) (3, -4) -- (4, -5.5);
        \draw (-10.5, 2.5) -- (-9, 2) -- (2, 2) -- (-.75, -2.125) -- (-.75, -5.5);
        \draw (5.25, 6.5) -- (4, 4) -- (6.5, 5.25);
        \node at (-2, 4) {$f_1 - g_3$};
        \node at (-7, -1) {$f_3 - g_3$};
        \node at (-4, -4) {$f_3 - g_1$};
        \node at (1, -5) {$f_3 - g_2$};
        \node at (1.5, -2) {$f_1 - g_2$};
        \node at (5, 1) {$f_2 - g_2$};
        \node at (4, 6) {$f_2 - g_3$};
        \draw[draw=none, fill=gray, fill opacity=.2] (-10.5, -5.5) -- (-10.5, 2.5) -- (-9, 2) -- (2, 2) -- (-.75, -2.125) -- (-.75, -5.5);
        \draw[draw=none, fill=gray, fill opacity=.2] (6.5, 6.5) -- (5.25, 6.5) -- (4, 4) -- (6.5, 5.25) -- (6.5, 6.5);
    \end{tikzpicture}
    \label{subfig:2d_example_plot}
    }
    \subfigure[Positions of \(P \cup N\) in the dual space. The first number indicates the \(z\)-coordinate. Dashed lines are projections of edges (1-cells) of the upper convex hull. The point \(\cR^{-1} (g_1)\) is fully below the hull.]{
    \begin{tikzpicture}
        \draw[->] (-2.5, 0) -- (2.5, 0);
        \draw[->] (0, -1.5) -- (0, 2.5);
        \draw (-2, -.1) -- (-2, .1) (-1, -.1) -- (-1, .1) (1, -.1) -- (1, .1) (2, -.1) -- (2, .1);
        \draw (-.1, -1) -- (.1, -1) (-.1, 1) -- (.1, 1) (-.1, 2) -- (.1, 2);
        \node at (-1, 1) {$\circ$}; \node at (-.8, .55) {\small $4, f_1$};
        \node at (1, 1) {$\circ$}; \node at (1.3, 1.3) {\small $-2, f_2$};
        \node at (-2, -1) {$\circ$}; \node at (-2, -1.3) {\small $-1, f_3$};
        \node at (0, 0) {$\times$}; \node at (-.3, -.2) {\small $0, g_1$};
        \node at (2, -1) {$\times$}; \node at (2, -1.3) {\small $2, g_2$};
        \node at (-1, 2) {$\times$}; \node at (-1.5, 2) {\small $2, g_3$};
        \draw[dashed, double] (-1, 2) -- (1, 1) -- (2, -1) -- (-2, -1) -- (-1, 2) -- (-1, 1) -- (2, -1);
        \draw[dashed] (-2, -1) -- (-1, 1) (-1, 2) -- (2, -1);
    \end{tikzpicture}
    \label{subfig:2d_example_dual_points}
    }
    \hspace{5mm}
    \subfigure[Projection of the faces of \(\U (P \oplus N)\) on the \(xy\)-plane. The first number indicates the \(z\)-coordinate; \(f_i - g_j\) is a shorthand for the dual point \(\cR^{-1} (f_i) + \cR^{-1} (g_j)\). Two points lie fully below the hull.]{
    \begin{tikzpicture}[scale=.8]
        \draw[->] (-3.5, 0) -- (1, 0) (3, 0) -- (3.5, 0);
        \draw[->] (0, -2.5) -- (0, 3.5);
        \draw (-3, -.1) -- (-3, .1) (-2, -.1) -- (-2, .1) (-1, -.1) -- (-1, .1) (1, -.1) -- (1, .1) (2, -.1) -- (2, .1) (3, -.1) -- (3, .1);
        \draw (-.1, -2) -- (.1, -2) (-.1, -1) -- (.1, -1) (-.1, 1) -- (.1, 1) (-.1, 2) -- (.1, 2) (-.1, 3) -- (.1, 3);
        \filldraw (-1, 1) circle (1.5pt); \node at (-.8, .6) {\small $4, f_1 - g_1$};
        \filldraw (1, 0) circle (1.5pt); \node at (1.2, -.4) {\small $6, f_1 - g_2$};
        \filldraw (-2, 3) circle (1.5pt); \node at (-2, 3.3) {\small $6, f_1 - g_3$};
        \filldraw (1, 1) circle (1.5pt); \node at (1.2, 1.5) {\small $-2, f_2 - g_1$};
        \filldraw (3, 0) circle (1.5pt); \node at (3.1, -.7) {\small $0, f_2 - g_2$};
        \filldraw (0, 3) circle (1.5pt); \node at (1, 3) {\small $0, f_2 - g_3$};
        \filldraw (-2, -1) circle (1.5pt); \node at (-2.7, -1.3) {\small $-1, f_3 - g_1$};
        \filldraw (0, -2) circle (1.5pt); \node at (.9, -2.2) {\small $1, f_3 - g_2$};
        \filldraw (-3, 1) circle (1.5pt); \node at (-3.4, 1.6) {\small $1, f_3 - g_3$};
        \draw[dashed] (-3, 1) -- (-2, 3) -- (0, 3) -- (3, 0) -- (0, -2) -- (-2, -1) -- (-3, 1) -- (0, -2) (-2, 3) -- (1, 0) -- (0, -2) (1, 0) -- (3, 0);
    \end{tikzpicture}
    \label{subfig:2d_example_dual_sum}
    }
    \caption{Illustration of the results on a two-dimensional example.}
\end{figure}

\end{document}